\documentclass{article} 
\usepackage{iclr2025_conference,times}

\usepackage{amsmath,amsfonts,bm}
\usepackage{xspace}
\usepackage{enumerate}
\usepackage{enumitem}
\usepackage{xfrac}
\usepackage{wrapfig}
\usepackage{parskip}

\usepackage{subcaption}
\usepackage{algorithm,algpseudocode}
\usepackage{pifont}

\algrenewcommand{\algorithmiccomment}[1]{\hfill\ding{228} #1}
\usepackage[para,online,flushleft]{threeparttable}
\usepackage{adjustbox}

\newcommand{\alg}{\textcolor{black}{\mbox{\textsc{MaxInfoRL}}}\xspace}

\newcommand{\algsac}{\textcolor{black}{\mbox{\textsc{MaxInfoSAC}}}\xspace}

\newcommand{\algdrq}{\textcolor{black}{\mbox{\textsc{MaxInfoDrQ}}}\xspace}

\newcommand{\algdrqddpg}{\textcolor{black}{\mbox{\textsc{MaxInfoDrQv2}}}\xspace}

\def\eqref#1{equation~\ref{#1}}

\def\1{\bm{1}}

\def\vtheta{{\bm{\theta}}}
\def\va{{\bm{a}}}

\def\vf{{\bm{f}}}

\def\vs{{\bm{s}}}

\def\vu{{\bm{u}}}

\def\vw{{\bm{w}}}

\def\vy{{\bm{y}}}
\def\vz{{\bm{z}}}

\def\vpi{{\bm{\pi}}}
\def\vsigma{{\bm{\sigma}}}

\def\vtau{{\bm{\tau}}}

\DeclareMathAlphabet{\mathsfit}{\encodingdefault}{\sfdefault}{m}{sl}
\SetMathAlphabet{\mathsfit}{bold}{\encodingdefault}{\sfdefault}{bx}{n}

\def\gA{{\mathcal{A}}}

\def\gD{{\mathcal{D}}}

\def\gH{{\mathcal{H}}}

\def\gO{{\mathcal{O}}}

\def\gS{{\mathcal{S}}}
\def\gT{{\mathcal{T}}}

\def\gX{{\mathcal{X}}}

\newcommand{\E}{\mathbb{E}}

\newcommand{\R}{\mathbb{R}}

\newcommand{\KL}{D_{\mathrm{KL}}}

\NewDocumentCommand{\norm}{sm}{\IfBooleanTF{#1}{\|#2\|}{\left\| #2 \right\|}}

\DeclareMathOperator*{\argmax}{arg\,max}
\DeclareMathOperator*{\argmin}{arg\,min}

\usepackage{graphicx}
\usepackage{amsthm}
\usepackage[hidelinks]{hyperref}
\usepackage{url}
\usepackage[capitalise]{cleveref}
\usepackage{pifont}
\usepackage{mathtools}

\theoremstyle{plain}
\newtheorem{theorem}{Theorem}[section]

\newtheorem{lemma}[theorem]{Lemma}

\theoremstyle{definition}

\theoremstyle{remark}

\Crefname{assumption}{Assumption}{Assumptions} 
\crefname{assumption}{assumption}{assumptions} 

\Crefname{lemma}{Lemma}{Lemma} 
\crefname{lemma}{lemma}{lemma} 

\title{\alg: Boosting exploration in \\reinforcement learning through \\information gain maximization}

\author{
Bhavya Sukhija$^{*, 1}$, Stelian Coros$^1$ 
Andreas Krause$^1$, Pieter Abbeel$^2$, Carmelo Sferrazza$^2$ \\
ETH Zürich $^1$, UC Berkeley$^2$ \\
\texttt{\{sukhijab, scoros, krausea\}@ethz.ch}\\
\texttt{\{pabbeel, csferrazza\}@berkeley.edu}
}

\iclrfinalcopy 
\begin{document}

\maketitle
{\def\thefootnote{}\footnotetext{Open-source implementations: \url{https://sukhijab.github.io/projects/maxinforl/}}}

\begin{abstract}
Reinforcement learning (RL) algorithms aim to balance exploiting the current best strategy with exploring new options that could lead to higher rewards. Most common RL algorithms use undirected exploration, i.e., select random sequences of actions.
Exploration can also be directed using intrinsic rewards, such as curiosity or model epistemic uncertainty. However, effectively balancing task and intrinsic rewards is challenging and often task-dependent. In this work, we introduce a framework, \alg, for balancing intrinsic and extrinsic exploration. \alg steers exploration towards informative transitions, by maximizing intrinsic rewards such as the information gain about the underlying task. 
When combined with Boltzmann exploration, this approach naturally trades off maximization of the value function with that of the entropy over states, rewards, and actions. We show that our approach achieves sublinear regret in the simplified setting of multi-armed bandits. We then apply this general formulation to a variety of off-policy model-free RL methods for continuous state-action spaces, yielding novel algorithms that achieve superior performance across hard exploration problems and complex scenarios such as visual control tasks.
\end{abstract}
\vspace{-3mm}

\begin{figure}[h]
    \centering
    \subcaptionbox{Normalized average performance of \alg across several deep RL benchmarks on state-based tasks.}[0.45\textwidth]{\includegraphics[width=0.45\textwidth]{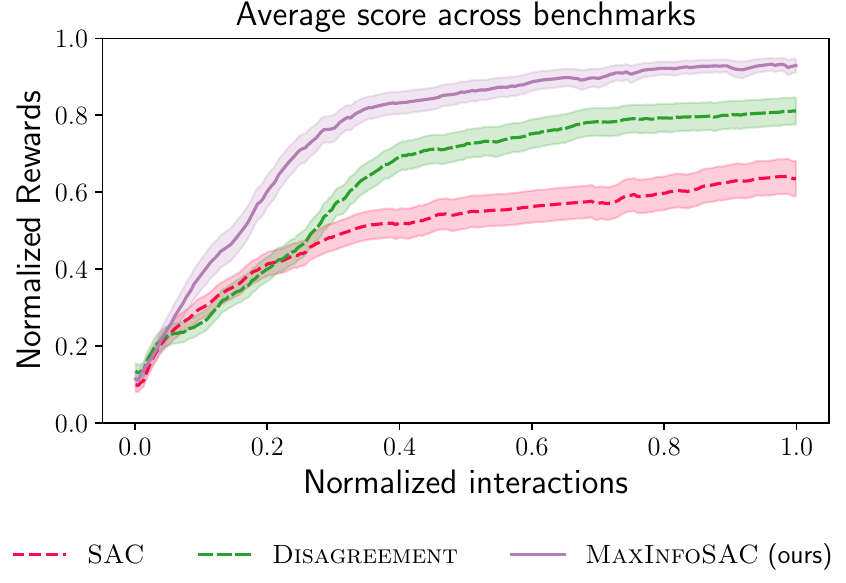}\label{fig:subfig1}}
    \hspace{4mm}
    \subcaptionbox{Normalized average performance of \algdrqddpg on the humanoid visual control tasks (stand, walk, and run).}[0.45\textwidth]{\includegraphics[width=0.45\textwidth]{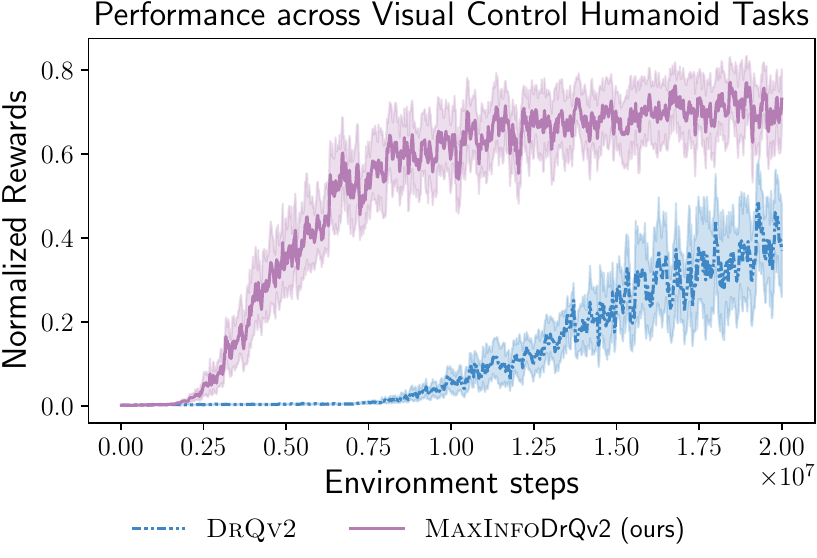}\label{fig:visual_control_summary}}
    \caption{We summarize the normalized performance of different variants of \alg; \algsac for state-based control and \algdrqddpg for visual control (cf.,~\cref{sec: experiments} for more details). We report the mean performance across five seeds with one standard error.}
    \label{fig:summary plot}
    \vspace{-3mm}
    \end{figure}
\section{Introduction} \label{sec:introduction}
\looseness=-1
Reinforcement learning (RL) has found numerous applications in sequential decision-making problems, from games~\citep{silver2017mastering}, robotics~\citep{hwangbo2019learning, brohan2023rt1roboticstransformerrealworld}, to fine-tuning of large language models~\citep{ouyang2022training}. However, most widely applied RL algorithms such as PPO~\citep{schulman2017proximal} are inherently sample-inefficient, requiring hundreds of hours of environment interactions for learning. Off-policy methods like SAC~\citep{haarnoja2018soft}, REDQ~\citep{chen2021randomized}, and DroQ ~\citep{hiraoka2021dropout} offer a more sample and compute efficient alternative and have demonstrated success in real-world learning~\citep{smith2022walk}. Despite this, they often require dense reward signals and suffer in the presence of sparse rewards or local optima. This is primarily due to their use of naive exploration schemes such as $\epsilon$-greedy or Boltzmann exploration and effectively take random sequences of actions for exploration.  These strategies are known to be suboptimal even for basic tasks~\citep{cesa2017boltzmann, burda2018exploration, sukhija2024optimistic}, yet remain popular due to their simplicity and scalability.

\looseness=-1
Several works~\citep{burda2018exploration, 
pathak2017curiosity,pathak2019self,sekar2020planning,sukhija2024optimistic}
 use intrinsic reward signals, e.g., curiosity or information gain, to improve the exploration of RL agents. Moreover, information gain~\citep{lidley_info_gain} is also widely applied in Bayesian experiment design~\citep{chaloner1995bayesian} and is the basis of many active learning methods~\citep{krause2008near, balcan2010true, hanneke2014theory, hübotter2024informationbased}. In RL, exploration via maximizing information gain offers strong theoretical guarantees~\citep{mania2020active, sukhija2024optimistic} and achieves state-of-the-art empirical performance~\citep{sekar2020planning, mendonca2021discovering}. However, a significant gap persists in both the theoretical and practical understanding of how to effectively balance intrinsic exploration objectives with naive extrinsic exploration algorithms. The goal of this work is to bridge this gap. To this end, we revisit the traditional, widely-used Boltzmann exploration, and enhance it by incorporating exploration bonuses derived from intrinsic rewards, like information gain. Our approach is grounded in both theoretical insights and practical motivations, and we empirically validate it across several deep RL benchmarks.

The key contributions of this work are summarized as follows:
\paragraph{Contributions}
\begin{enumerate}[leftmargin=0.5cm]
 \item We propose \alg, 
 a novel class of off-policy model-free algorithms for continuous state-action spaces that augments existing RL methods with directed exploration. In essence, \alg builds on standard Boltzmann exploration and guides it via an intrinsic reward. We propose a practical auto-tuning procedure that largely simplifies trading off extrinsic and intrinsic objectives. This yields algorithms that explore by visiting trajectories that achieve the maximum information gain about the underlying MDP, while efficiently solving the task. As a result, \alg retains the simplicity of traditional RL methods while adding directed exploration through intrinsic rewards. Additionally, we show how the same idea can be combined with other naive exploration techniques, such as $\epsilon$--greedy.
 \item In the simplified setting of stochastic multi-armed bandits in continuous spaces, we show that \mbox{\alg} has sublinear regret. In addition, we show that \alg benefits from all theoretical properties of contraction and convergence that hold for max-entropy RL algorithms such as SAC~\citep{haarnoja2018soft}.
 \item In our experiments, we 
 use an ensemble of dynamics models to estimate information gain and
 combine \alg with SAC~\citep{haarnoja2018soft}, REDQ~\citep{chen2021randomized}, DrQ~\citep{yarats2021image}, and DrQv2~\citep{yarats2021mastering}.
 We evaluate it on standard deep RL benchmarks for state and visual control tasks and show that \alg performs the best across all tasks and baselines, obtaining the highest performance also in challenging exploration problems (see.,~\cref{fig:summary plot} for the average performance of \alg across several environments). 
\end{enumerate}

\section{Background}\label{sec: background}
\looseness -1 A core challenge in RL is deciding whether the agent should leverage its current knowledge to maximize rewards or try new actions in pursuit of better solutions. Striking this balance between exploration--exploitation is critical. Here, we first introduce the problem setting, then we discuss two of the most commonly used exploration strategies in RL: $\epsilon$-greedy and Boltzmann exploration.

\subsection{Problem Setting} \label{sec: problem setting}
We study an infinite-horizon Markov decision process~\citep[MDP,][]{puterman2014markov}, defined by the tuple $(\gS, \gA, p, \gamma, r, \rho)$, where the state and action spaces are continuous, i.e., $\gS \subset \R^{d_s}, \gA \subset \R^{d_a}$, and the
unknown transition kernel $p: \gS \times \gS \times \gA \to [0, \infty) $ represents the probability density of the next state
$\vs_{t+1} \in \gS$ given the current state $\vs_{t} \in \gS$ and action $\va_{t} \in \gA$. At each step $t$ in the environment, the agent observes the state $\vs_t$, samples an action $\va_t$ from the policy $\vpi: \gA \times \gS \to [0, \infty)$, $\va_t \sim \vpi(\va|\vs_t)$, and 
receives a reward $r: \gS \times \gS \times \gA \to [-\frac{1}{2}r_{\max}, \frac{1}{2}r_{\max}]$. The agent's goal is to learn a policy $\vpi^*$ that maximizes the $\gamma$ discounted reward w.r.t.~the initial state distribution $\vs_0 \sim \rho$. 
 \begin{equation}
    \label{eq:rl-objective}
    \vpi^* = \argmax_{\vpi \in \Pi} J(\vpi) = \argmax_{\vpi \in \Pi} \E_{\vs_0, \va_0, \dots} \left[\sum_{t = 0}^{\infty} \gamma^t r_t \right].
\end{equation}

\looseness=-1
In the following, we provide the definitions of the state-action value function $Q^{\vpi}$ and the value function $V^{\pi}$:
\begin{equation*}
    Q^{\vpi}(\vs_t, \va_t) = \E_{\vs_{t+1}, \va_{t+1} \sim \vpi, \dots}  \left[\sum_{l = 0}^{\infty} \gamma^l r_{t+l} \right], \;
    V^{\vpi}(\vs_t) = \E_{\va_t \sim \vpi, \vs_{t+1}, \va_{t+1} \sim \vpi, \dots}  \left[\sum_{l = 0}^{\infty} \gamma^l r_{t+l} \right].
    \label{eq: vf_defs}
\end{equation*}
\vspace{-3mm}
\subsection{$\epsilon$--greedy and exploration } \label{sec: bg_eps_greedy}
The $\epsilon$-greedy strategy~\citep{kearns2002near, mnih2013playing, van2016deep} is widely applied in RL to balance exploration and exploitation, where the RL agent follows this simple decision rule below to select actions
\begin{equation}
    \va_t = \begin{cases}
   \va \sim \text{Unif}(\gA) \; & \text{with probability }\epsilon_t \\
    \underset{\va \in \gA}{\arg\max} \; Q^{*}(\vs_t, \va) \; & \text{else},
    \end{cases}
    \label{eq: eps_greedy_orig}
\end{equation}
Here $Q^{*}$ is the estimate of the optimal state-action value function.
Therefore, at each step $t$,  with probability $\epsilon_t$, a random action $\va_t \sim \text{Unif}(\gA)$ is sampled, else the
greedy action $\va_t = \max_{\va \in \gA} Q^{*}(\vs_t, \va)$ is picked.  \citet{lillicrap2015continuous, fujimoto2018addressing} extend this strategy to continuous state-action spaces, where a deterministic policy $\vpi_{\theta}$ is learned to maximize the value function and combined with random Gaussian noise for exploration. 

\subsection{Boltzmann Exploration}
Boltzmann exploration is the basis of many RL algorithms~\citep{sutton2018reinforcement, szepesvari2022algorithms}. The policy distribution $\vpi$ for Boltzmann is represented through
\begin{equation}
    \vpi(\va|\vs) \propto \exp\left(\alpha^{-1}Q^{\vpi}(\vs, \va)\right),
   \label{eq: bolztman standard}
\end{equation}
where $\alpha$ is the temperature parameter that regulates exploration and $Q^{\vpi}$ is the soft-$Q$ function. We neglect the normalization term $Z^{-1}(\vs)$ in the definition for simplicity. As $\alpha \to 0$, the policy greedily maximizes  $Q^{\vpi}(\vs, \va)$, i.e. it exploits, and as $\alpha \to \infty$ the policy adds equal mass to all actions in $\gA$, effectively performing uniform exploration. Intuitively, Boltzmann exploration can be interpreted as a smoother alternative to $\epsilon$--greedy, with $\alpha$ serving a similar role to $\epsilon$ in controlling the degree of exploration.
\cite{cesa2017boltzmann} show that the standard Boltzmann exploration is suboptimal even in the simplest settings. They highlight that a key shortcoming of Boltzmann exploration is that it does not reason about the uncertainty of its estimates.

Overall, both $\epsilon$--greedy and Boltzmann exploration 
strategies are undirected. They fail to account for the agent's ``lack of knowledge'' and do not encourage risk- or knowledge-seeking behavior. The agent explores by sampling random action sequences, which leads to suboptimal performance, particularly in challenging exploration tasks with continuous state-action spaces.

\subsection{Intrinsic exploration with information gain}
\looseness=-1
Intrinsic rewards or motivation are used to direct agents toward underexplored regions of the MDP. Hence they enable RL agents to acquire information in a more principled manner as opposed to the aforementioned naive exploration methods. 
 Effectively, the agent explores by selecting policies that maximize the $\gamma$-discounted intrinsic rewards.   
A common choice for the intrinsic reward is the information gain~\citep{elementsofIT, sekar2020planning, mendonca2021discovering, sukhija2024optimistic}. 
Accordingly, for the remainder of the paper, we center our derivations around using information gain as the intrinsic reward. However, our approach is flexible and can also be combined with other intrinsic exploration objectives, such as RND (\cite{burda2018exploration}, see~\cref{sec: additional experiments}).

 We study a non-linear dynamical system of the form
\begin{equation}
    \tilde{\vs}_{t+1} = \vf^*(\vs_t, \va_t) + \vw_t.
\end{equation}
Here $\tilde{\vs}_{t+1} = [\vs^{\top}_{t+1}, r_t]^{\top}$ represents the next state and reward, $\vf^*$ represents the \emph{unknown dynamics} and reward function of the MDP and $\vw_t$ is the process noise, which we assume to be zero-mean i.i.d.,$~\sigma^2$-Gaussian. Note this is a very common representation of nonlinear systems with continuous state-action spaces~\citep{khalil2015nonlinear} and the basis of many RL algorithms~\citep{pathak2019self, kakade2020information, curi2020efficient, mania2020active, wagenmaker2023optimal, sukhija2024neorl}. Furthermore, it models all essential and unknown components of the underlying MDP; the transition kernel and the reward function. 

\paragraph{Approximating information gain} Given a dataset of transitions $\gD_n = \{(\vs_i, \va_i, \tilde{\vs}'_i)\}^n_{i=0}$, e.g., a replay buffer, we learn a Bayesian model of the unknown function $\vf^*$, to obtain a posterior distribution $p(\vf^*|\gD_n)$ for $\vf^*$. This distribution can be Gaussian, e.g., Gaussian process models~\citep{rasmussen2005gp} or represented through Bayesian neural networks like probabilistic ensembles~\citep{lakshminarayanan2017simple}. As opposed to the typical model-based RL setting, similar to~\citet{burda2018exploration, pathak2017curiosity, pathak2019self}, our learned model is only used to determine the intrinsic reward. 
The information gain $I(\tilde{\vs}'; \vf^* | \vs, \va, \gD_n)$, reflects the uncertainty about the unknown dynamics $\vf^*$ from observing the transition $(\vs, \va, \tilde{\vs}')$.
Moreover, 
let $\vsigma(\vs, \va|\gD_n) = [\sigma_{j} (\vs, \va)]_{j\leq d_{\vs} +1}$ denote the model epistemic uncertainty or disagreement of $\vf^*$. 
\citet[Lemma 1.]{sukhija2024optimistic} show that
\begin{equation}
I(\tilde{\vs}'; \vf^* | \vs, \va, \gD_n) =  H(\tilde{\vs}' | \vs, \va, \gD_n) - H(\tilde{\vs} | \vs, \va, \vf^*, \gD_n) \leq \underbrace{\sum^{d_{\vs}+1}_{j=1} \log \left(1 + \frac{\sigma_{n-1, j}^2(\vs_t, \va_t)}{\sigma^2}\right)}_{I_u(\vs, \va)} \label{eq: entropy bound}
\end{equation}
where $H$ denotes the 
(differential) entropy~\citep{elementsofIT} and
in \cref{eq: entropy bound} the equality holds when $p(\vf^*|\gD_n)$ is Gaussian. 
Note that while the above is an upper bound, \cite{sukhija2024optimistic} motivate this choice from a theoretical perspective proving convergence of the active learning algorithm for the model-based setting. In this work, similar to \citet{pathak2019self, sekar2020planning, mendonca2021discovering, sukhija2024optimistic}, we use the upper bound of the information gain for our practical algorithm. The upper bound has a natural interpretation, since by picking actions $\va_t$ that maximize it, we effectively visit areas where we have high uncertainty about the unknown function $\vf^*$, therefore performing exploration in both state and action space. Empirically, this approach has shown to perform well, e.g.,~\citet{sekar2020planning, mendonca2021discovering}.

\paragraph{Data dependence of intrinsic rewards} Information gain and other intrinsic rewards depend on the data $\gD_n$, making them inherently nonstationary and non-Markovian. Intuitively, underexplored areas of the MDP become less informative once visited (c.f.,~\citet{prajapat2024submodular} for more details). However, in RL, intrinsic rewards are often treated similarly to extrinsic rewards, a simplification that works very well in practice~\citep{burda2018exploration, sekar2020planning}. We take a similar approach in this paper and omit the dependence of $I$ on $\gD_n$ and use $I(\tilde{\vs}'; \vf^* | \vs, \va)$ from hereon for simplicity. 

\section{\alg} \label{sec: method}
In this section, we present our method for combining intrinsic exploration with classical exploration strategies. While \alg builds directly on Boltzmann exploration, we begin by illustrating its key ideas in the context of an $\epsilon$--greedy strategy, due to its mathematical simplicity and natural distinction between exploration and exploitation steps. 
The insights gained from this serve as motivation for developing our main method: \alg with Boltzmann exploration algorithms, which we evaluate in \Cref{sec: experiments}.

\subsection{Modifying $\epsilon$--greedy for directed exploration} \label{sec: eps max info}
We modify the $\epsilon$--greedy strategy from \cref{sec: bg_eps_greedy} and learn two critics, $Q_{\text{extrinsic}}^{*}$ and $Q_{\text{intrinsic}}^{*}$, where $Q_{\text{extrinsic}}^{*}$ is the state-action value function of the extrinsic reward $r$ and $Q_{\text{intrinsic}}^{*}$ the critic of an intrinsic reward function $r_{\text{intrinsic}}$, for instance, the information gain (see~\cref{eq: entropy bound}). 
Unlike traditional 
$\epsilon$--greedy exploration, we leverage intrinsic rewards to guide exploration more effectively by selecting actions that maximize $Q_{\text{intrinsic}}^{*}$,
leading to more informed exploration rather than random sampling.
At each step $t$, we pick a greedy action that maximizes $Q_{\text{extrinsic}}^{*}$ with probability $1-\epsilon_t$, while for exploration, the action that maximizes the intrinsic critic is selected, i.e., $\va_t = \max_{\va \in \gA} Q_{\text{intrinsic}}^{*}(\vs_t, \va)$. 
\begin{equation}
    \va_t = \begin{cases}
    \underset{\va \in \gA}{\arg\max} \; Q_{\text{intrinsic}}^{*}(\vs_t, \va) \; & \text{with probability } \epsilon_t \\
    \underset{\va \in \gA}{\arg\max} \; Q_{\text{extrinsic}}^{*}(\vs_t, \va) \; & \text{else},
    \end{cases}
    \label{eq: eps_greedy}
\end{equation}

\looseness=-1
We call the resulting exploration strategy $\epsilon$--\alg.
This approach is motivated by the insight that in continuous spaces, intrinsic rewards cover the state-action spaces much more efficiently than undirected random exploration, making them more effective for exploration in general~\citep{aubret2019survey, sekar2020planning, sukhija2024optimistic}. 
In \Cref{sec: eps_greedy_bandits}, to give a theoretical intuition of our approach, we study $\epsilon$--\alg in the simplified setting of multi-armed bandit (MAB). We show that as more episodes are played, it gets closer to the optimal solution, i.e. has sublinear-regret. %

The key takeaway from $\epsilon$--\alg is that instead of exploring with actions that maximize entropy in the action spaces, e.g., uniform sampling, we select policies that also yield high information about the MDP during learning. In the following, we leverage this idea and modify the target distribution of Boltzmann exploration to incorporate intrinsic exploration bonuses. Moreover, 
$\epsilon$--\alg has two practical drawbacks; (\emph{i}) it requires training two actor-critics and (\emph{ii}) practically, the probability $\epsilon_t$ is specified by the problem designer. 
We address both these limitations in the section below and present our main method. 

\subsection{\alg with Boltzmann exploration}\label{subsec:boltzmanexplore}



In \cref{sec: eps max info}, we modify $\epsilon$--greedy to sample actions with high intrinsic rewards during exploration instead of randomly picking actions. Motivated from the same principle, 
we augment the distribution of Boltzmann exploration with the intrinsic reward $I(\tilde{\vs}'; \vf^* | \vs, \va)$ to get the following
\begin{equation}
    \vpi(\va|\vs) \propto \exp\left(\alpha^{-1}Q^{\vpi}(\vs, \va)  + I(\tilde{\vs}'; \vf^* | \vs, \va)\right).
    \label{eq: bolztman max info}
\end{equation}
 The resulting distribution encourages exploration w.r.t.~information gain, with $\alpha$ playing a similar role to $\epsilon$ in \cref{eq: eps_greedy}. 
 Therefore, \cref{eq: bolztman max info} can be viewed as a \textit{soft} formulation of \cref{eq: eps_greedy}.
 Effectively, instead of randomly sampling actions, for large values of the temperature, we 
pick actions that yield high information
while maintaining the exploitative behavior for smaller temperatures.
This distribution is closely related to the epistemic risk-seeking exponential utility function from $K$-learning~\citep{o2021variational} and probabilistic inference in RL~\citep{tarbouriech2024probabilistic}. 
As we show in the following, this choice of parameterization results in a very intuitive objective for the policy. Given the previous policy $\vpi^{\text{old}}$ and $Q^{\vpi^{\text{old}}}$, akin to~\citet{haarnoja2018soft}, we select the next policy $\vpi^{\text{new}}$ through the following optimization
\begin{align}
\vpi^{\text{new}} &= \argmin_{\vpi \in \Pi} \text{D}_{\text{KL}}\left(\vpi(\cdot|\vs) \Bigg\| Z^{-1}(\vs)\exp\left(\frac{1}{\alpha}Q^{\vpi^{\text{old}}}(\vs, \cdot) +  I(\tilde{\vs}'; \vf^* | \vs, \va)\right)\right) \notag \\
 &= \argmax_{\vpi \in \Pi}  \E_{\va \sim \vpi(\cdot|\vs)}\left[ Q^{\vpi^{\text{old}}}(\vs, \va)  - \alpha \log(\vpi(\va|\vs)) + \alpha I(\tilde{\vs}'; \vf^* | \vs, \va)\right] \notag \\
     &= \argmax_{\vpi \in \Pi}  \E_{\va \sim \vpi(\cdot|\vs)}\left[ Q^{\vpi^{\text{old}}}(\vs, \va)\right]  + \alpha H(\tilde{\vs}', \va | \vs),
      \label{eq: policy_update}
\end{align}
here in the last line we used that $\E_{\va \sim \vpi(\cdot|\vs)}[-\log(\vpi(\va|\vs)) + I(\tilde{\vs}'; \vf^* | \vs, \va)] = H(\va|\vs) + H(\tilde{\vs}'|\va, \vs) - H(\tilde{\vs}'|\vs, \va, \vf^*) = H(\tilde{\vs}', \va|\vs) - H(\vw)$.
Hence, the policy $\vpi^{\text{new}}$ trades off maximizing the value function with the \emph{entropy of the states, rewards, and actions}. This trade-off is regulated through the temperature parameter $\alpha$. We provide a different perspective to \cref{eq: policy_update} from the lens of control as inference~\citep{levine2018reinforcement, hafner2020action} in \cref{sec: kl minimization}. 
\paragraph{Separating exploration bonuses} 
\alg has two exploration bonuses; (\emph{i}) the policy entropy, and (\emph{ii}) the information gain (\cref{eq: entropy bound}).  The two terms are generally of different magnitude and tuning the temperature for the policy entropy is fairly well-studied in RL~\citep{haarnoja2018soft}. To this end, we modify \cref{eq: policy_update} 
and introduce two individual temperature parameters $\alpha_1$ and $\alpha_2$ to separate the bonuses. Furthermore, since information gain does not have a closed-form solution in general, akin to prior work~\citep{sekar2020planning, sukhija2024optimistic}, we use its upper bound $I_u(\vs, \va)$ (\cref{eq: entropy bound}) instead.
\begin{align}J^{\vpi^{\text{old}}}(\vpi|\vs)  &=
  \E_{\va \sim \vpi(\cdot|\vs)}\left[Q^{\vpi^{\text{old}}}(\vs, \va)
 -\alpha_1 \log(\vpi(\va|\vs)) + \alpha_2 I_u(\vs, \va)\right]
     \notag \\
\vpi^{\text{new}}(\cdot|\vs)
  &= \argmax_{\vpi \in \Pi} J^{\vpi^{\text{old}}}(\vpi|\vs) \label{eq:vf_soft_q_updated}
\end{align}
\looseness -1 
For $\alpha_1$, we can either use a deterministic policy with $\alpha_1 = 0$ like~\citet{lillicrap2015continuous} or auto-tune  $\alpha_1$ as suggested by \citet{haarnoja2018soft}. Notably, for $\alpha_2=0$ we get the standard max entropy RL methods~\citep{haarnoja2018soft}. Therefore, by introducing two separate temperatures, 
we can treat information gain as another exploration bonus in addition to the policy entropy and combine it with any RL algorithm.
\paragraph{Auto-tuning the temperature for the information gain bonus}

\citet{haarnoja2018soft} formulate the problem of soft-Q learning as a constrained optimization.
\begin{align*}
   \vpi^*(\cdot| \vs) &\coloneqq \underset{\vpi \in \Pi}{\arg\max}\; \E_{\va \sim \vpi}\left[ Q^{\vpi}(\vs, \va)\right] \; \text{s.t.,} \; H(\va| \vs) \geq \bar{H} \\
   &\coloneqq \underset{\vpi \in \Pi}{\arg\max}\; \min_{\alpha_1 \geq 0} \E_{\va \sim \vpi}\left[ Q^{\vpi}(\vs, \va) - \alpha_1 (\log(\vpi(\va|\vs)) + \bar{H})\right].
\end{align*}
The entropy coefficient is then auto-tuned by solving this optimization problem gradually via stochastic gradient descent (SGD).
In a similar spirit,
we propose the following constraints to auto-tune the temperatures for the entropy and the information gain
\begin{align*}
   \vpi^*(\cdot| \vs) &\coloneqq \underset{\vpi \in \Pi}{\arg\max}\; \E_{\va \sim \vpi}\left[ Q^{\vpi}(\vs, \va)\right] \; \text{s.t.,} \; H(\va| \vs) \geq \bar{H}, \E_{\va \sim \vpi} \left[I_u(\vs, \va)\right] \geq \bar{I}_u(\vs) \addtocounter{equation}{1}\tag{\theequation}    \label{eq: constraint}\\
   &\coloneqq \underset{\vpi \in \Pi}{\arg\max} \min_{\alpha_1, \alpha_2 \geq 0} \E\left[ Q^{\vpi}(\vs, \va) - \alpha_1 (\log(\vpi(\va|\vs)) + \bar{H}) + \alpha_2 (I_u(\vs, \va) - \bar{I}_u(\vs)) \right]
\end{align*}
\citet{haarnoja2018soft} use a simple heuristic $\bar{H} = - \text{dim}(\gA)$ for the target entropy. However, we cannot specify a general desired information gain since this depends on the learned Bayesian model $p(\vf^*)$. This makes choosing $\bar{I}_u$ task-dependent. For our experiments, we maintain a target policy $\bar{\vpi}$, updated similarly to a target critic in off-policy RL, and define $\bar{I}_u$ as
\begin{equation}
\bar{I}_u(\vs) \coloneqq \sum^{d_{\vs}+1}_{j=1} \E_{\va \sim \bar{\vpi}(\cdot|\vs)}\left[\log \left(1 + \sigma^{-2}\sigma_{n-1, j}^2(\vs, \va)\right)\right]
\label{eq: target info gain}
\end{equation}
Intuitively, the constraint enforces that the current policy $\vpi$ explores w.r.t.~the information gain, at least as much as the target policy $\bar{\vpi}$. 
In principle, any other constraint can be used to optimize for $\alpha_2$. We consider our constraint since (\emph{i}) it is easy to evaluate, (\emph{ii}) it can be combined with other intrinsic rewards\footnote{$I$ could represent a different intrinsic reward function, e.g., RND~\citet{burda2018exploration}.}, and (\emph{iii}) it is modular, i.e., it can be added to any RL algorithm. 
Moreover, as \alg can be combined with any base off-policy RL algorithm such as SAC~\citep{haarnoja2018soft} or DDPG~\citep{lillicrap2015continuous}, it benefits from the simplicity and scalability of these methods. In addition, it introduces the information gain as a directed exploration bonus and automatically tunes its temperature similar to the policy entropy in \citet{haarnoja2018soft}.
Therefore, it benefits from both the strengths of the naive extrinsic exploration methods and the directedness of intrinsic exploration.
We demonstrate this in our experiments, where we combine \alg with SAC~\citep{haarnoja2018soft}, REDQ~\citep{chen2021randomized}, DrQ~\citep{yarats2021image}, and DrQv2~\citep{yarats2021mastering}. 

\paragraph{Convergence of \alg} \label{sec: theory boltzmann}
In the following, we study our modified Boltzmann exploration strategy and show that
as in~\citet{haarnoja2018soft}, the update rules for $Q$ function and the policy converge to an optimal policy $\vpi^* \in \Pi$. We make a very general assumption that the entropy of the policy and the model epistemic uncertainty are all bounded for all $(\vs, \va) \in \gS \times \gA$. The proof of the theorem and the related lemmas are provided in \cref{sec: boltzmann proofs}.

We define the Bellman operator $\gT^{\vpi}$
\begin{equation}
    \gT^{\vpi} Q(\vs, \va) = r(\vs, \va) + \gamma \E_{\vs'|\vs, \va}[ V^{\vpi}(\vs')],
    \label{eq:bellman_operator}
\end{equation}
where 
\begin{equation}
    V^{\vpi}(\vs) = \E_{\va \sim \vpi(\cdot|\vs)}[Q(\vs, \va) -\alpha_1 \log(\vpi(\va|\vs)) + \alpha_2 I_u(\vs, \va)]
     \label{eq:vf_soft_q}
\end{equation}
is the soft-value function.

\begin{theorem}[\alg soft Q learning]
     Assume that the reward, the entropy for all $\vpi \in \Pi$, and the model epistemic uncertainty $\vsigma_n$ are all bounded for all $n\geq 0$, $(\vs, \va) \in \gS \times \gA$. 
    The repeated application of soft policy evaluation (\cref{eq:bellman_operator}) and soft policy update (\cref{eq:vf_soft_q_updated}) to any $\vpi \in \Pi$ converges to $\vpi^* \in \Pi$ such that $Q^{\vpi}(\vs, \va)  \leq Q^{\vpi^{*}}(\vs, \va)$ for all $\vpi \in \Pi$, $(\vs, \va) \in \gS \times \gA$.
    \label{thm: boltzman convergence}
\end{theorem}
\Cref{thm: boltzman convergence} shows that our reformulated expression for Boltzmann exploration exhibits the same convergence properties from~\citet{haarnoja2018soft}. 

\section{Experiments} \label{sec: experiments}
\looseness=-1
We evaluate \alg with Boltzmann exploration from \cref{subsec:boltzmanexplore} across several deep RL benchmarks~\citep{brockman2016openai, tassa2018deepmind, Sferrazza-RSS-24} on state-based and visual control tasks. 
In all our experiments, we report the mean performance with standard error evaluated over five seeds. 
For the state-based tasks we combine \alg with SAC~\citep{haarnoja2018soft} and for the visual control tasks with DrQ~\citep{yarats2021image} and DrQv2~\citep{yarats2021mastering}. 
In the following, we refer to these algorithms as \algsac, \algdrq, and \algdrqddpg, respectively. To further demonstrate the generality of \alg, in \Cref{sec: additional experiments}, we provide additional experiments, where we combine \alg with REDQ~\citep{chen2021randomized}, OAC~\citep{ciosek2019better}, DrM~\citep{xu2023drm}, use RND~\citep{burda2018exploration} as an intrinsic reward instead of the information gain, and also evaluate the $\epsilon$--\alg from \cref{sec: eps max info} with curiosity~\citep{pathak2017curiosity} and disagreement~\citep{pathak2019self, sekar2020planning, mendonca2021discovering} as intrinsic rewards. 

\looseness=-1
\textbf{Baselines}: For the state-based tasks, in addition to SAC, we consider four baselines, all of which use SAC as the base RL algorithm:
\begin{enumerate}[leftmargin=0.5cm]
    \item \textsc{Disagreement}: Employs an explore then exploit strategy, where it maximizes only the intrinsic reward for the first $25\%$ of environment interaction and then switches to the exploitation phase where the extrinsic reward is maximized. 
    We use disagreement in the forward dynamics model for the intrinsic reward. 
    \item \textsc{Curiosity}: The same as \textsc{Disagreement} but with curiosity as the intrinsic reward. 
    \item \textsc{SACIntrinsic}: Based on~\citet{burda2018exploration}, where a normalized intrinsic reward is added to the extrinsic reward. 
    \item \textsc{SACEipo}: Uses a weighted sum of intrinsic and extrinsic rewards, where the weight is tuned with the extrinsic optimality constraint from~\citet{chen2022redeeming}. 
\end{enumerate}
For the visual control tasks, we use DrQ and DrQv2 as our baselines. More details on our baselines and experiment details are provided in \cref{sec: Experiment details}.

\textbf{Does \algsac achieve better performance on state-based control problems?} 
\looseness=-1
\looseness -1 In \Cref{fig:maxinforl_comparison} we compare \algsac with the baselines across several tasks of varying dimensionality\footnote{including the humanoid from DMC: $d_s = 67$, $d_a = 21$} from the DeepMind control~\citep[DMC,][]{tassa2018deepmind} and OpenAI gym benchmark suite~\citep{brockman2016openai}. Across all tasks we observe that \algsac consistently performs the best. While the other baselines perform on par with \algsac on some tasks, they fail to solve others. On the contrary, \algsac consistently achieves the highest performance in all environments (cf.,~\cref{fig:summary plot}). To further demonstrate the scalability of \algsac, we evaluate it on a practical robotics benchmark, namely 
HumanoidBench~\citep{Sferrazza-RSS-24} that features a simulated Unitree H1 robot on a variety of tasks. We use the \textit{no-hand} version of the benchmark, and evaluate our algorithm on the stand, walk, and run tasks. We compare \algsac to SAC in \cref{fig:maxinforl_humanoind_bench} and again observe that \algsac achieves overall higher performance, except for a minor convergence delay on the stand task, which is a trivial exploration problem concerning pure stabilization. 
\begin{figure}[th]
    \centering
    \includegraphics[width=\linewidth]{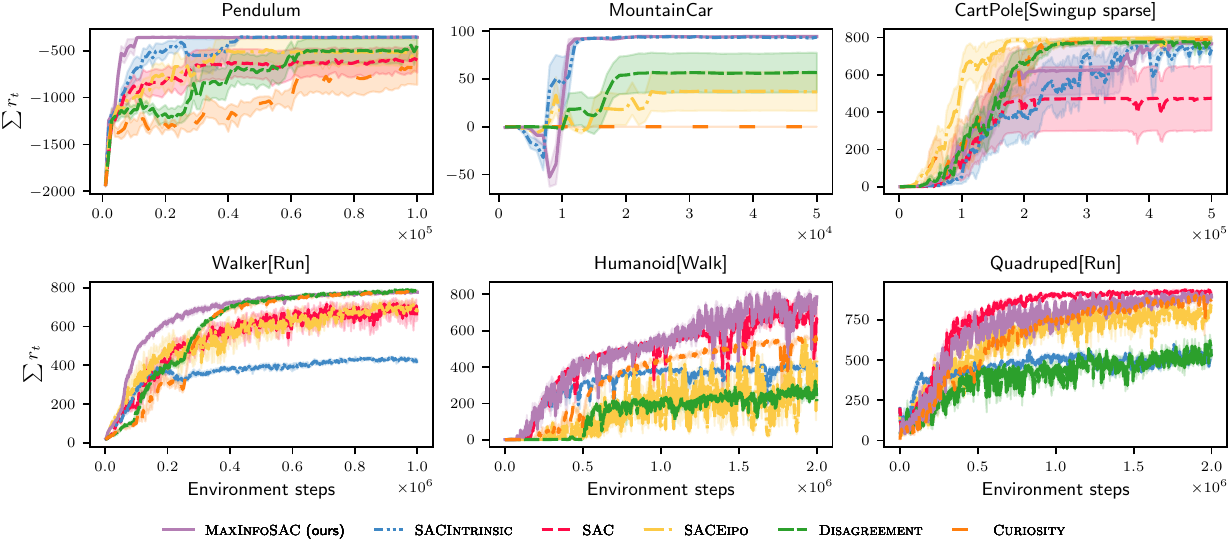}
    \caption{Learning curves of all methods on several environments from the OpenAI gym and DMC suite benchmarks.}
    \label{fig:maxinforl_comparison}
    \vspace{-3mm}
\end{figure}

\begin{figure}[th]
    \vspace{-3mm}
    \centering
    \includegraphics[width=\linewidth]{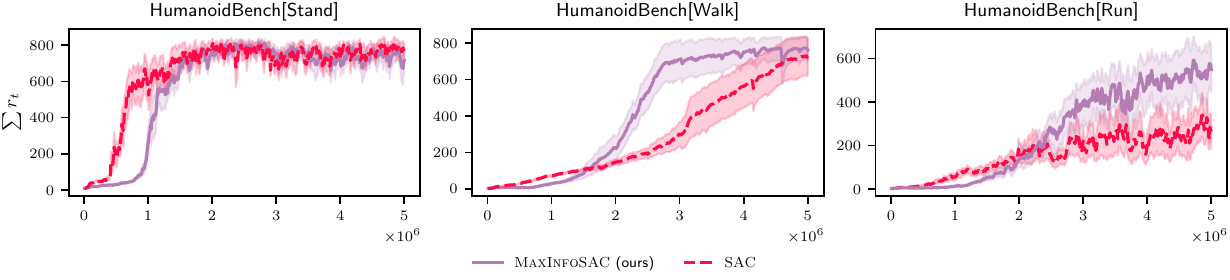}
    \vspace{-6mm}
    \caption{Performance of \algsac and SAC on the HumanoidBench benchmark.}
    \vspace{-5mm}
    \label{fig:maxinforl_humanoind_bench}
\end{figure}

\textbf{Does \algsac solve hard exploration problems?}
\looseness=-1
Naive exploration techniques often struggle with challeging exploration tasks~\citep{burda2018exploration, curi2020efficient, sukhija2024optimistic}. To test \algsac in this context, similar to \citet{curi2020efficient}, we modify the reward in Pendulum, CartPole, and Walker by adding an action cost, $r_{\text{action}}(\va) = -K \norm{\va}_2$, where $K$ controls the penalty for large actions. \citet{curi2020efficient} empirically show that even for small $K$ values, naive exploration methods fail, converging to the sub-optimal solution of applying no actions.

\looseness=-1
In \Cref{fig:maxinforl_action_cost} we compare \algsac with the baselines. We observe that SAC struggles with action costs, especially in CartPole and Walker. Both \textsc{SACEipo} and \textsc{SACIntrinsic} underperform, likely due to poor handling of extrinsic and intrinsic rewards. Specifically, \textsc{SACEipo} quickly reduces its intrinsic reward weight to zero (cf.,~\cref{fig:saceipo weights} in \cref{sec: additional experiments}), making it overly greedy. Disagreement and curiosity-based methods perform better since we manually tune their number of intrinsic exploration interactions. However, \algsac achieves the best performance by automatically balancing intrinsic and extrinsic rewards. For \algsac and SAC, we also depict the phase plot from the exploration on the pendulum environment in~\Cref{fig:pendulum phase plot}. \algsac covers the state space much faster than SAC, effectively solving the Pendulum swing-up task ($\text{Target} = (0, 0)$) within $10$K environment steps.

\begin{figure}[th]
    \vspace{-1mm}
    \centering
    \includegraphics[width=\linewidth]{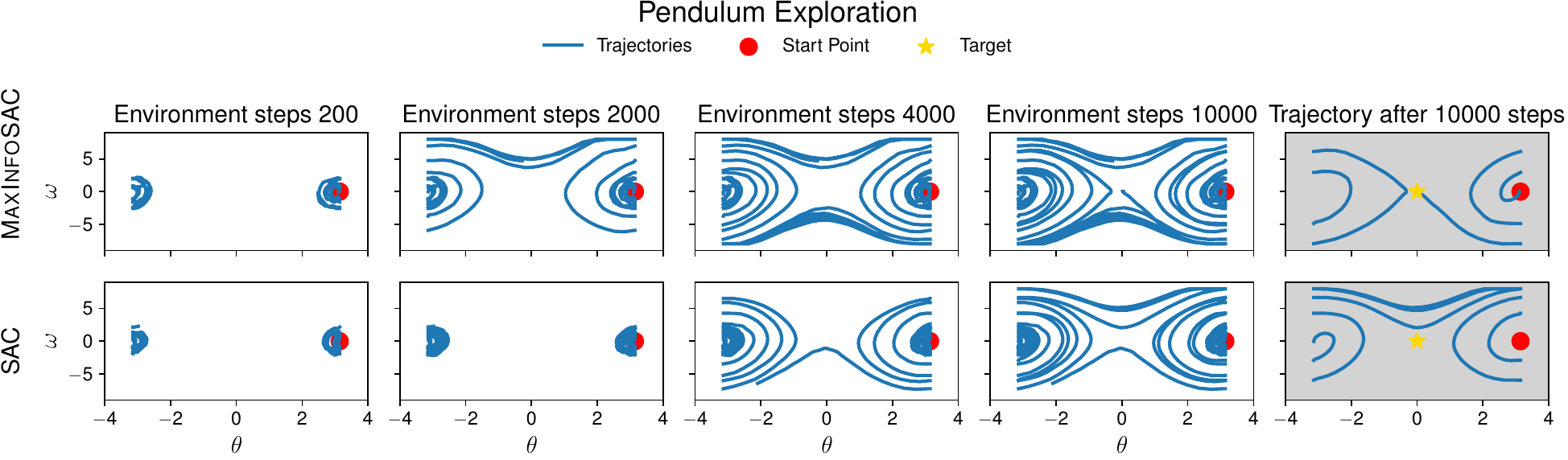}
    \vspace{-5mm}
    \caption{Phase plots during learning of \alg and SAC on the Pendulum environment. \algsac covers the state space much faster and effectively solves the swing-up task within $10$K environment interactions.}
    \label{fig:pendulum phase plot}
    \vspace{-2mm}
\end{figure}

\begin{figure}[th]
    \vspace{-2mm}
    \centering
    \includegraphics[width=\linewidth]{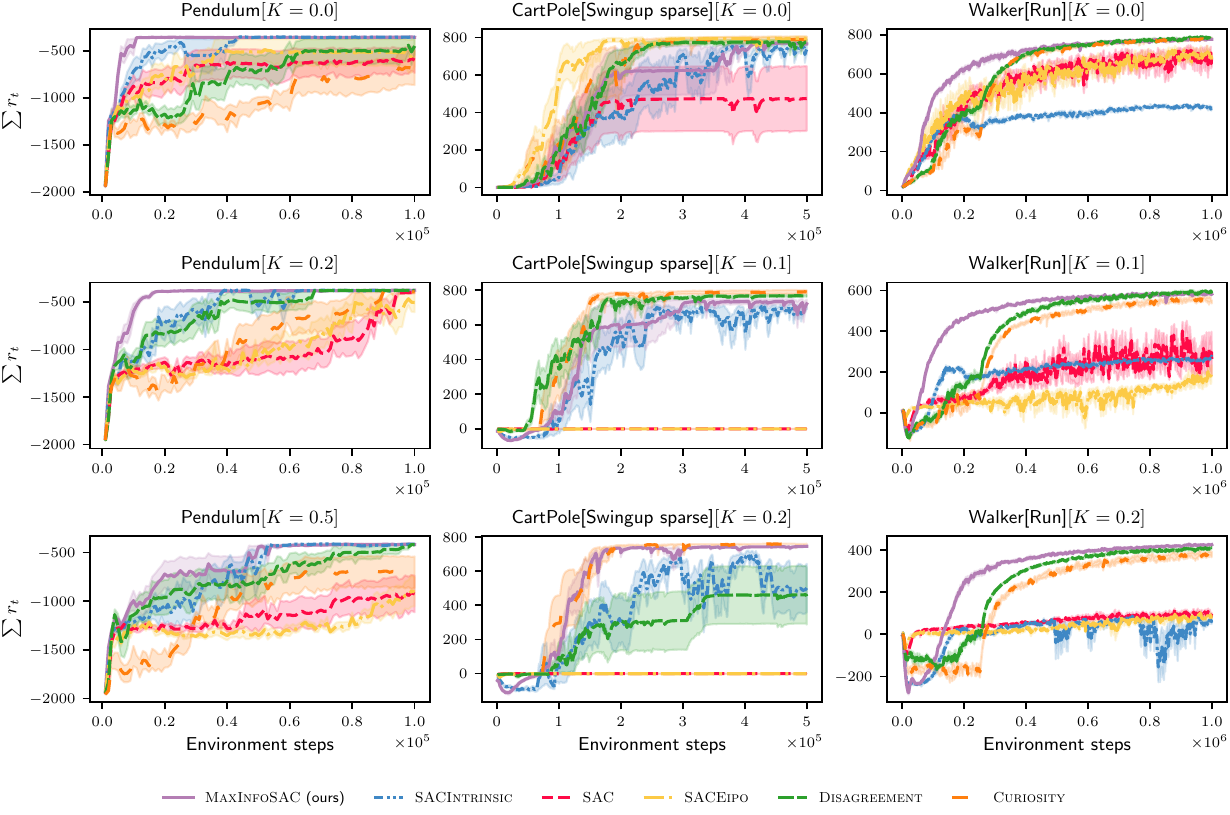}
    \caption{
    \vspace{-3mm}
    Learning curves for state-based tasks for different values of the action cost parameter $K$.}
    \label{fig:maxinforl_action_cost}
    \vspace{-0mm}
\end{figure}

\textbf{Does \alg scale to visual control tasks?}
\begin{figure}[th]
    \centering
    \includegraphics[width=\linewidth]{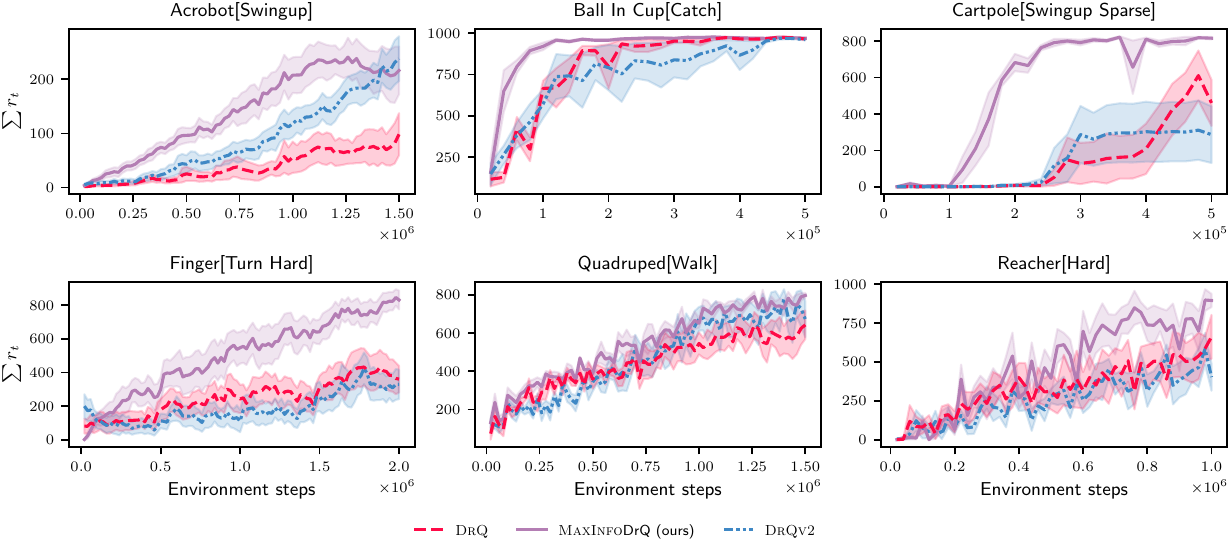}
    \vspace{-5mm}
    \caption{Learning curves from visual control tasks of the DMC suite.}
    \label{fig:maxinforl_drq}
    \vspace{-3mm}
\end{figure}
\looseness=-1
In this section, we combine \alg with DrQ and evaluate it on visual control problems from the DeepMind control suite~\citep[DMC,][]{tassa2018deepmind}. DrQ is a visual control algorithm based on the max entropy framework from SAC, therefore, it can easily be combined with \alg.
We call the resulting algorithm \algdrq. 
We compare \algdrq with DrQ and DrQv2 in \Cref{fig:maxinforl_drq}. From the figure, we conclude that \algdrq consistently reaches higher rewards and better sample efficiency than the baselines across all tasks. This illustrates the scalability and generality of \alg.

\textbf{Solving challenging visual control tasks with \alg}
For challenging visual control problems, in particular, the humanoid tasks from DMC, \citet{yarats2021mastering} propose DrQv2, a modified version of DrQ, which uses $n$--step returns and DDPG with noise scheduling instead of SAC for the base algorithm. They claim that switching to DDPG with noise scheduling is particularly useful for solving the humanoid tasks. To this end, we combine \alg with DrQv2 (\algdrqddpg) and evaluate it on the stand, walk, and run humanoid tasks from DMC. This demonstrates the flexibility of \alg, as it can seamlessly be combined with DrQv2, cf.,~\cref{sec: Experiment details} for more details and \cref{sec: additional experiments} for the performance of \algdrqddpg on other DMC tasks. We compare \algdrqddpg with DrQv2 in \Cref{fig:maxinforl_humanoid}. \algdrqddpg results in substantial gains in performance and sample efficiency compared to DrQv2. To our knowledge, those shown in \Cref{fig:maxinforl_humanoid} are the highest returns reached in these challenging visual control tasks by model-free RL algorithms in the literature. 
This further illustrates the advantage of directed exploration via information gain/intrinsic rewards. 

\begin{figure}[th]
    \vspace{-2mm}
    \centering
    \includegraphics[width=\linewidth]{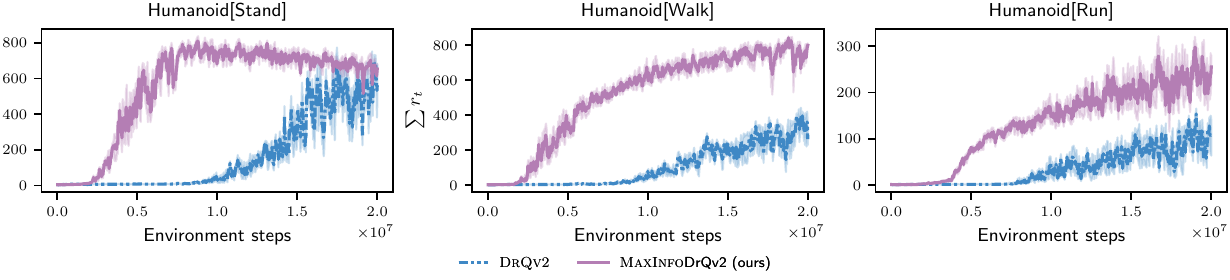}
    \caption{Learning curves from the visual control humanoid tasks of the DMC suite.}
    \label{fig:maxinforl_humanoid}
    \vspace{-3mm}
\end{figure}

\section{Related Works} \label{sec: related works}
\paragraph{Naive Exploration}
\looseness=-1
Naive exploration approaches such as $\epsilon$--greedy or Boltzmann are widely applied in RL due to their simplicity~\citep{mnih2013playing, schulman2017proximal,
lillicrap2015continuous,
haarnoja2018soft, hafner2023mastering}. 
In particular, the maximum entropy framework~\citep{ziebart2008maximum} is the basis of many sample-efficient model-free deep RL algorithms~\citep{haarnoja2018soft, chen2021randomized, hiraoka2021dropout, yarats2021image}.
However, these methods often perform suboptimally, especially in challenging exploration problems such as those with sparse rewards or local optima (cf.,~\cref{sec: experiments}). 
Effectively, the agent explores the underlying MDP by taking random sequences of actions. In continuous spaces, this makes sufficiently covering the state and action space exceptionally challenging.
Moreover, even in the simplest setting of MAB in continuous spaces, the most common and theoretically sound exploration strategies are Thompson sampling (TS), upper-confidence bound (UCB)~\citep{srinivas, chowdhury2017kernelized}, and information-directed sampling~\citep{russo2018learning, kirschner2021information}. There are several RL algorithms based on thes strategies~\citep{brafman2002r, jaksch10a, 
ouyang2017learning, nikolov2018information, ciosek2019better, kakade2020information, curi2020efficient, sukhija2024neorl}. Similarly, there are methods from Bayesian RL~\citep{osband2018randomized, fellows2021bayesian, buening2023minimax} that perform principled exploration. However, the naive exploration approaches remain ubiquitous in deep RL due to their simplicity.
Instead, intrinsic rewards are often used to facilitate more directed exploration. However, how to combine extrinsic and intrinsic exploration is much less understood both theoretically and practically. 

\paragraph{Intrinsic exploration}
Several works use intrinsic rewards as a surrogate objective to facilitate exploration in challenging tasks \citep[cf.,][for a comprehensive survey]{aubret2019survey}. 
Common choices of intrinsic rewards are model prediction error or ``Curiosity''~\citep{schmidhuber1991possibility, pathak2017curiosity, burda2018exploration}, novelty of transitions/state-visitation counts~\citep{stadie2015incentivizing,bellemare2016unifying}, diversity of skills/goals~\citep{eysenbach2018diversity, DADS, nair2018visual, skewfit}, empowerment~\citep{klyubin2005empowerment, salge2014empowerment}, state-entropy~\citep{mutti2021task, seo2021state, kim2024accelerating},
    and information gain over the forward dynamics~\citep{sekar2020planning, mendonca2021discovering, sukhija2024optimistic}.  In this work, we focus on the information gain since it is the basis of many theoretically sound and empirically strong active learning methods~\citep{krause2008near, balcan2010true, hanneke2014theory, mania2020active, sekar2020planning, sukhija2024optimistic, hübotter2024informationbased}. Furthermore, we also motivate the choice of information gain/model epistemic uncertainty theoretically (cf.,~\cref{thm: epsilon_greedy_thm} and \cref{sec: kl minimization}). Nonetheless, \alg can also be used with other intrinsic exploration objectives. In this work, we study how to combine the intrinsic rewards with the extrinsic ones for exploration in RL. Moreover, approaches such as~\cite{pathak2017curiosity, pathak2019self, sukhija2024optimistic, sekar2020planning} execute an explore then exploit strategy, where initially the extrinsic reward is ignored and a policy is trained to maximize the intrinsic objective. After this initial exploration phase, the agent is trained to maximize the extrinsic reward. This strategy is common in active learning methods and does not result in sublinear cumulative regret even for the simple MAB setting. This is because the agent executes a finite number of exploration steps instead of continuously trading off exploration and exploitation. To this end, \cite{burda2018exploration, chen2022redeeming} combine the extrinsic and intrinsic exploration rewards by taking a weighted sum $r_{\text{tot}} = \lambda r_{\text{extrinsic}} + r_{\text{intrinsic}}$. While~\citet{burda2018exploration} use fixed weight $\lambda$, \citet{chen2022redeeming} propose the extrinsic optimality constraint to auto-tune the intrinsic reward weight. The extrinsic optimality constraint enforces that the policy maximizing the sum of intrinsic and extrinsic rewards performs at least as well as the policy which only maximizes the extrinsic ones. Keeping a fixed weight $\lambda$ across all tasks and environment steps might be suboptimal and the constraint from \citet{chen2022redeeming} may quickly lead to the trivial greedy solution $\lambda^{-1} = 0$ (cf.,~\cref{fig:saceipo weights} in \cref{sec: additional experiments}).  We compare \alg with these two methods, as well as with explore then exploit strategies. We show that \alg outperforms them across several deep RL benchmarks (cf.,~\cref{sec: experiments}). 

\section{Conclusion} \label{sec: conclusion}
\looseness=-1
In this work, we introduced \alg, a class of model-free off-policy RL algorithms that train an agent to trade-off reward maximization with an exploration objective that targets high entropy in the state, action, and reward spaces. The proposed approach is theoretically sound, simple to implement, and can be combined with most common RL algorithms. \alg consistently outperforms its baselines in a multitude of benchmarks and achieves state-of-the-art performance on challenging visual control tasks.

A limitation of \alg is that it requires training an ensemble of forward dynamics models to compute the information gain, which increases computation overhead (c.f.~\cref{sec: Experiment details} \cref{tab:compute cost}). In addition, while effective, the constraint in \Cref{eq: constraint} also requires maintaining a target policy.

\looseness=-1
The generality of the \alg exploration bonus and the trade-off we propose with reward maximization are not limited to the off-policy model-free setting, which we 
focus on here due to their sample and computational efficiency. 
Future work will investigate 
its applicability to other classes of RL algorithms, such as model-based RL, where the forward dynamics model is generally part of the training framework. In fact, in this work we only use the forward dynamics model for the intrinsic rewards. Another interesting direction is to include samples from the learned model in the policy training, which may yield additional gains in sample efficiency. Lastly, extending our theoretical guarantees from the bandit settings to the MDP case is also an interesting direction for future work.

\subsubsection*{Acknowledgments}

The authors would like to thank Kevin Zakka for the helpful discussions. Bhavya Sukhija was gratefully supported by ELSA (European Lighthouse on Secure and Safe AI) funded by the European Union under grant agreement No. 101070617.
This project was supported in part by the ONR Science of Autonomy Program N000142212121, the Swiss National Science Foundation under NCCR Automation, grant agreement 51NF40 180545, the Microsoft Swiss Joint Research Center, the SNSF Postdoc Mobility Fellowship 211086, and an ONR DURIP grant. Pieter Abbeel holds concurrent appointments as a Professor at UC Berkeley and as an Amazon Scholar. This paper describes work performed at UC Berkeley and is not associated with Amazon.

\bibliography{iclr2025_conference}
\bibliographystyle{iclr2025_conference}

\appendix
\newpage
\section{Analyzing $\epsilon$--Greedy for Multi-armed Bandits} \label{sec: eps_greedy_bandits}
In this section, we enunciate a theorem that shows that our modified $\epsilon$--greedy approach from \cref{sec: eps max info} has sublinear regret in the simplified setting of multi-arm bandits. 
\subsection{Sublinear regret for $\epsilon$--\alg}
This exploration--exploitation trade-off is also fundamental in multi-armed bandits (MAB)~\citep{lattimore2020bandit}, where
the task is to optimize an unknown objective function $J: \Theta \to [-\frac{1}{2}J_{\max}, \frac{1}{2}J_{\max}]$, with $\Theta \subset \R^{d}$ being a compact set over which we seek the optimal arm $\vtheta^* = \argmax_{\vtheta \in \Theta} J(\vtheta)$.  In each round $t$, we
choose a point $\vtheta_t \in \Theta$ and obtain a noisy measurement $y_t$ of the function value, that is, $y_t = J(\vtheta_t) + w_t$. Our goal is
to maximize the sum of rewards $\sum^T_{t=1} J(\vtheta_t)$ over $T$ learning iterations/episodes, thus to
perform essentially as well as $\vtheta^*$
(as quickly as possible). For example,
$\vtheta$ can be parameters of the policy distribution $\vpi_{\vtheta}$ and our objective maximizing the discounted rewards as defined in \cref{eq:rl-objective}, i.e., $\max_{\vtheta \in \Theta} J(\vpi_{\vtheta})$. 
This formulation is the basis of several sample-efficient RL algorithms~\citep{Calandra2016,DBLP:journals/corr/MarcoHBST16,DBLP:journals/corr/AntonovaRA17, safeopt, sukhija2023gosafeopt}. 
\looseness=-1

\looseness=-1
The natural performance metric in this context is the cumulative regret defined as $R_T = \sum^T_{t=1} J(\vtheta^*) - J(\vtheta_t)$.
A desirable asymptotic property of any learning algorithm is that it is no-regret: $\lim_{T \to \infty} R_T / T = 0$. In the following theorem, we show that under standard continuity assumptions on $J$, i.e., we can model it through Gaussian process regression~\citep{rasmussen2005gp} with a given kernel $k(\cdot, \cdot): \Theta \times \Theta \to \R_{+}$, our exploration strategy from \cref{eq: bolztman max info} with the model epistemic uncertainty as the intrinsic objective has sublinear regret.
\begin{theorem}
    Assume that the objective function $J$ lies in the RKHS $\gH_k$ corresponding
to the kernel $k(\vtheta, \vtheta)$,  $\norm{J}_k \leq B$, and
that the noise $\vw_t$ is zero-mean $\sigma$-sub Gaussian. Let $\epsilon_t = t^{2\alpha - 1}$ for $\alpha \in (\sfrac{1}{4}, \sfrac{1}{2})$, $\delta \in (0, 1]$ and consider the following $\epsilon$--greedy exploration strategy
\begin{equation*}
    \vtheta_t = \begin{cases}
    \underset{\vtheta \in \Theta}{\arg\max} \; \sigma_t(\vtheta) \; & \text{with probability} \epsilon_t \\
    \underset{\vtheta \in \Theta}{\arg\max}  \; \mu_t(\vtheta) \; & \text{else},
    \end{cases}
\end{equation*}
where $\mu_t$ is our mean estimate of $J$ and $\sigma_t$ the epistemic uncertainty (c.f.,~\citep{rasmussen2005gp} for the exact formula).
Then we have with probability at least $1-\delta$
\begin{equation*}
    R_T \leq \gO\left(J_{\max}T^{2\alpha} + 2 \Gamma_T(k) T^{1-\alpha} \right),
\end{equation*}
where $\Gamma_T(k)$ is the maximum information gain~\citep{srinivas}.
\label{thm: epsilon_greedy_thm}
\end{theorem}
The maximum information gain $\Gamma_T(k)$ measures the complexity of learning the function $J$ w.r.t.~the number of data points $T$ and depends on the choice of the kernel $k$, for common kernels such as the RBF and linear kernel, it grows polylogarithmically with $T$~\citep{srinivas, vakili2021information}, resulting in a sublinear regret bound.
Intuitively, by sampling informative actions sufficiently often, we will improve our estimate of $J$ and thus gradually suffer less regret.

\subsection{Proof of \Cref{thm: epsilon_greedy_thm}}
From hereon, let $Z_t \sim \text{Bernoulli}(\epsilon_t)$ be a Bernoulli random variable. The acquisition function for $\epsilon$-greedy GP bandit from \cref{thm: epsilon_greedy_thm} is defined as:
\begin{equation}
    \vtheta_t = \begin{cases}
    \underset{\vtheta \in \Theta}{\arg\max} \; \sigma_t(\vtheta) \; & \text{if } Z_t = 1 \\
    \underset{\vtheta \in \Theta}{\arg\max}  \; \mu_t(\vtheta) \; & \text{else}
    \end{cases}
    \label{eq: acquisition function}
\end{equation}

Let $Q_t = \sum^{t}_{s=1} Z_s - \epsilon_s$. In the following, we analyze the sequence $Q_{1:t}$. Note that $Q_{t} \in [0, t]$ for all $t$ and
\begin{align*}
    \E[Q_{t} | Q_{1:t-1}] &= \E[Z_t - \epsilon_t] + \E[Q_{t-1}|Q_{1:t-1}] \\
    &= Q_{t-1}.
\end{align*}
Therefore, $Q_{1:t}$ is a martingale sequence, or equivalently, $\{Z_s - \epsilon_s\}_{s \in 1:t}$ is a martingale difference sequence.

In the following, we use the time-uniform Azuma Hoeffding inequality from \citet{kassraie2024anytime} to give a bound on $\sum^T_{s=1} Z_s$. 
\begin{lemma}
    Let $\epsilon_1 = 1$, then we have with probability at least $1-\delta$, for all $t \geq 1$
    \begin{equation*}
        \sum^{t}_{s=1} Z_s \geq \max\left\{\sum^{t}_{s=1} \epsilon_s - \frac{5}{2} \sqrt{t (\log\log(et) + \log(2\delta))}, 1\right\}
    \end{equation*}
\end{lemma}
\begin{proof}
Note that $|Q_{t} - Q_{t-1}| = |Z_t - \epsilon_t| \in [0, 1]$ for all $t$. Therefore, we can use the time-uniform Azzuma-Hoeffding inequality~\cite[Lemma 26.]{kassraie2024anytime} to get
\begin{equation*}
    \Pr\left(\exists t: \sum^{t}_{s=1} \epsilon_s - Z_s \geq \frac{5}{2} \sqrt{t (\log\log(et) + \log(2\delta))}\right) \leq \delta 
    \end{equation*}
    Therefore, we have with probability at least $1-\delta$
    \begin{equation*}
        \sum^{t}_{s=1} Z_s \geq \sum^{t}_{s=1} \epsilon_s - \frac{5}{2} \sqrt{t (\log\log(et) + \log(2\delta))}
    \end{equation*}
    for all $t \geq 1$.
    Furthermore, since $\epsilon_1 = 1$, we have $Z_1 = 1$. Therefore, 
    \begin{equation*}
        \sum^{t}_{s=1} Z_s \geq \max\left\{\sum^{t}_{s=1} \epsilon_s - \frac{5}{2} \sqrt{t (\log\log(et) + \log(2\delta))}, 1\right\}
    \end{equation*}
    
\end{proof}

\looseness=-1
Let $E_t$ denote the number of exploration steps after $t$ iterations of the algorithm, that is, $E_t = \sum^t_{s=1} Z_s$. In the next lemma, we derive a bound on $E_t$.
\begin{lemma}
    Let $\epsilon_t = t^{2\alpha - 1}$ with $\alpha \in (\sfrac{1}{4}, \sfrac{1}{2})$, then for any $\delta \in (0, 1]$ exists a $C > 0, t_0 > 0$ such that $E_t \geq C^2 t^{2\alpha}$ for all $t \geq t_0$ with probability at least $1-\delta$.
\end{lemma}
\begin{proof}
Note that $\epsilon_t$ is monotonously decreasing with $t$. Therefore, $\int^{s+1}_{s} x^{2\alpha -1} dx \leq \epsilon_s$ and hence
\begin{equation*}
    \sum^t_{s=1} \epsilon_s \geq \int^{t}_{1} x^{2\alpha -1} dx = \frac{t^{2\alpha}-1}{2\alpha}
\end{equation*}
Therefore, we get
    \begin{equation*}
        \sum^{t}_{s=1} \epsilon_s - \frac{5}{2} \sqrt{t (\log\log(et) + \log(2\delta))} \geq  \frac{t^{2\alpha}}{2\alpha} - \left(\frac{5}{2} \sqrt{t (\log\log(et) + \log(2\delta))} + \frac{1}{2\alpha}\right)
    \end{equation*}
    Therefore, 
    \begin{align*}
        E_t &= \sum^{t}_{s=1} Z_s \\
        &\geq \max\left\{\sum^{t}_{s=1} \epsilon_s - \frac{5}{2} \sqrt{t (\log\log(et) + \log(2\delta))}, 1\right\} \\
        &\geq \max\left\{\frac{t^{2\alpha}}{2\alpha} - \left(\frac{5}{2} \sqrt{t (\log\log(et) + \log(2\delta))} + \frac{1}{2\alpha}\right), 1\right\}
    \end{align*}

Note that since $\alpha > \frac{1}{4}$, for any $\delta \in (0, 1]$ 
\begin{equation*}
    \frac{5}{2} \sqrt{t (\log\log(et) + \log(2\delta))} + \frac{1}{2\alpha} \in o(t^{2\alpha}).
\end{equation*}
Moreover $\max\left\{\frac{t^{2\alpha}}{2\alpha} - \left(\frac{5}{2} \sqrt{t (\log\log(et) + \log(2\delta))} + \frac{1}{2\alpha}\right), 1\right\} \in \Theta(t^{2\alpha})$, therefore for any $\delta \in (0, 1]$ exists a $C > 0, t_0 > 0$ such that $E_t \geq C^2 t^{2\alpha}$ for all $t \geq t_0$ with probability at least $1-\delta$.

\end{proof}
Next, we use the results from \citet{chowdhury2017kernelized} on the well-calibration of $J$ w.r.t.~our mean and epistemic uncertainty estimates. 
\begin{lemma}[Theorem 2, \cite{chowdhury2017kernelized}]
Let the assumptions from~\cref{thm: epsilon_greedy_thm} hold.
Then, with
probability at least $1 - \delta$, the following holds for all $\vtheta \in \Theta$ and $t \geq 1$:
\begin{equation*}
    |\mu_{t-1}(\vtheta) - J(\vtheta)| \leq \beta_{t-1} \sigma_{t-1}(\vtheta),
    \end{equation*}
with $\beta_{t-1} = B + \sigma
\sqrt{
2(\Gamma_{t-1} + 1 + \ln(\sfrac{1}{\delta}))}$,
where $\Gamma_{t-1}$ is the maximum information gain~\citep{srinivas} after $t-1$ rounds.
\label{lem: chowdhury}
\end{lemma}

\begin{proof}[Proof of \cref{thm: epsilon_greedy_thm}]
The $\epsilon$--greedy algorithm alternates between uncertainty sampling and maximizing the mean. In the worst case, it may suffer maximal regret $J_{\max}$ during the uncertainty sampling stage. Furthermore, assume at time $t$, we execute the exploitation stage, then we have
\begin{align*}
    J(\vtheta^*) - J(\vtheta_t) &= J(\vtheta^*) - \mu_{t}(\vtheta^*) + \mu_{t}(\vtheta^*) - J(\vtheta_t) \\
    &\leq J(\vtheta^*) - \mu_{t}(\vtheta^*) + \mu_{t}(\vtheta_t) - J(\vtheta_t) \tag{$\mu_{t}(\vtheta_t) \geq \mu_{t}(\vtheta^*)$} \\
    &\leq \beta_t (\sigma_{t}(\vtheta^*) + \sigma_{t}(\vtheta_t)) \tag{\Cref{lem: chowdhury}}
\end{align*}
Therefore, we can break down the cumulative regret $R_T$ into these two stages.
    \begin{align*}
        R_T &= \E_{Z_{1:T}, w_{1:T}}\left[\sum^T_{t=1} J(\vtheta^*) - J(\vtheta_t)\right] \\
        &=\sum^T_{t=1} [\E_{Z_{1:t-1}, w_{1:t-1}}\E_{w_t, Z_t}[J(\vtheta^*) - J(\vtheta_t)]] \\
        &= \sum^T_{t=1} \E_{Z_{1:t-1}, w_{1:t-1}}[\E_{w_t}[J(\vtheta^*) - J(\vtheta_t)|Z_t=0](1-\epsilon_t) + \E_{w_t}[J(\vtheta^*) - J(\vtheta_t)|Z_t=1]\epsilon_t \\
        &\leq J_{\max}\sum^T_{t=1} \epsilon_t + \sum^T_{t=1} \beta_t \E_{Z_{1:t-1}, w_{1:t-1}}[\sigma_t(\vtheta^*) + \sigma_t(\vtheta_t)]
    \end{align*}
    Assume $E_t$ is given, 
    then we have for all $t$ 
    \begin{align*}
        \max_{\vtheta \in \Theta} \sigma^2_t(\vtheta) E_t 
        &= \max_{\vtheta \in \Theta} \sigma^2_t(\vtheta) \sum^{t}_{s=1} Z_s \\ 
        &\leq \sum^{t}_{s=1} Z_s \max_{\vtheta \in \Theta} \sigma^2_s(\vtheta) \\
        &\leq \sum^{t}_{s=1}Z_s \sigma^2_s(\vtheta_s) \\
        &\leq \sum^{t}_{s=1} \sigma^2_s(\vtheta_s) \\
        &\leq C \Gamma_t
    \end{align*}
    Therefore, $\sigma^2_t(\vtheta) \leq \frac{C \Gamma_t}{E_t}$ for all $\vtheta \in \gX$. Since, $E_{t} \geq C^2 t^{2\alpha}$ for $\alpha \in \left(\frac{1}{4}, \frac{1}{2}\right)$, we have with probability at least $1-\delta$ for all $t \geq t_0$, $\vtheta \in \Theta$, there exists a constant $K$ such that
    \begin{equation*}
        \sigma_t(\vtheta) \leq K\frac{\sqrt{\Gamma_t}}{t^{\alpha}}
    \end{equation*}
    Going back to the regret, we get
    \begin{align*}
         R_t &\leq J_{\max}\sum^T_{t=1} \epsilon_t + \sum^T_{t=1} 2\beta_t K\frac{\sqrt{\Gamma_t}}{t^{\alpha}} \\
         &\leq \gO\left(J_{\max}T^{2\alpha} + 2K\beta_T \sqrt{\Gamma_T} T^{1-\alpha} \right)
    \end{align*}
    Here we used the fact that, $\alpha \in \left(\frac{1}{4}, \frac{1}{2}\right)$, therefore, 
    \begin{align*}
        \sum^T_{t=1} 2\beta_t K\frac{\sqrt{\Gamma_t}}{t^{\alpha}} \leq 2 K\beta_T \sqrt{\Gamma_T} \sum^T_{t=1} \frac{1}{t^{\alpha}} \leq \Theta(2 K\beta_T \sqrt{\Gamma_T} T^{1-\alpha}).
    \end{align*}
    Similarly since $\epsilon_t = t^{2\alpha-1}$, we have $\sum^T_{t=1} \epsilon_t = \Theta(T^{2\alpha})$
The regret is sublinear for ubiquitous choices of kernels such as the RBF. Indicating, consistency of the $\epsilon$-greedy exploration strategy with the model-epistemic uncertainty as the exploration signal.
    
\end{proof}
\section{Proof of \Cref{thm: boltzman convergence}} \label{sec: boltzmann proofs}
\begin{lemma}
Consider the Bellman operator $\gT^{\vpi}$ from \cref{eq:bellman_operator}, and a mapping $Q^{0}: \gS\times \gA \to \R$ and define $Q^{k+1} = \gT^{\vpi}Q^{k}$. 
Furthermore, let the assumptions from \cref{thm: boltzman convergence} hold. Then, the sequence $Q^{k}$ will converge to the soft-Q function of $\vpi$ as $k \to \infty$. 
\label{lemm: contraction q function}
\end{lemma}
\begin{proof}
    Consider the following augmented reward \begin{equation*}
        r_{\vpi}(\vs, \va) = r(\vs, \va) + \E_{\vs'|\vs, \va} \left[\E_{\va' \sim \vpi(\cdot|\vs')}[-\alpha_1 \log(\vpi(\va'|\vs')) + \alpha_2 I(\tilde{\vs}''; \vf^* | \vs', \va')]\right],
    \end{equation*}
    where $\tilde{\vs}''$ is the next state and reward given $(\vs', \va')$.
    Then from \cref{eq:bellman_operator}, 
    we have 
    \begin{equation*}
         \gT^{\vpi} Q(\vs, \va) = r_{\vpi}(\vs, \va) + \gamma \E_{\vs', \va'|\vs, \va}[ Q(\vs', \va')].
    \end{equation*}
    Furthermore, $r_{\vpi}$ is bounded since $r$ is bounded and the policy entropy and the information gain are bounded for all $\vpi \in \Pi$ by assumption. Therefore, we can apply standard convergence results on policy evaluation to show convergence~\citep{sutton2018reinforcement}.
\end{proof}
\Cref{lemm: contraction q function} shows that the Q updates satisfy the contraction property and thus converge to the soft-Q function of $\vpi$. Next, we show that the policy update from \cref{eq: policy_update} leads to monotonous improvement of the policy.
\begin{lemma}
    Let the assumptions from~\Cref{thm: boltzman convergence} hold.
    Consider any $\vpi^{\text{old}} \in \Pi$, and let $\vpi^{\text{new}}$ denote the solution to \cref{eq: policy_update}. Then we have for all $(\vs, \va) \in \gS \times \gA$ that $Q^{\vpi^{\text{old}}}(\vs, \va)  \leq Q^{\vpi^{\text{new}}}(\vs, \va)$.
    \label{lem: policy_improvement}
\end{lemma}
\begin{proof}
    Consider any $\vs \in \gS$
    \begin{align}
        &\E_{\va \sim \vpi^{\text{new}}(\cdot|\vs)}\left[ Q^{\vpi^{\text{old}}}(\vs, \va)  - \alpha_1 \log \vpi^{\text{new}}(\va|\vs) + \alpha_2 I(\vs'; \vf^* |\vs, \va)\right] \notag \\
        &\geq \E_{\va \sim \vpi^{\text{old}}(\cdot|\vs)}\left[ Q^{\vpi^{\text{old}}}(\vs, \va)  - \alpha_1 \log \vpi^{\text{old}}(\va|\vs) + \alpha_2 I(\vs'; \vf^* |\vs, \va)\right] \notag\\
        &= V^{\vpi^{\text{old}}}(\vs) \label{eq: bound_q}
    \end{align}
    Next consider any pair $(\vs, \va) \in \gS \times \gA$
    \begin{align*}
        Q^{\vpi^{\text{old}}}(\vs, \va) &= r(\vs, \va) + \gamma \E_{\vs'|\vs, \va}[ V^{\vpi^{\text{old}}}(\vs')] \\
        &\leq r(\vs, \va) + \gamma \E_{\vs'|\vs, \va}[\E_{\va' \sim \vpi^{\text{new}}(\cdot|\vs')} [Q^{\vpi^{\text{old}}}(\vs', \va')  - \alpha_1 \log \vpi^{\text{new}}(\va'|\vs') + \alpha_2 I(\vs''; \vf^* |\vs', \va')]] \\
        &\vdots \\
        &\leq Q^{\vpi^{\text{new}}}(\vs, \va)
    \end{align*}
    where we have repeatedly expanded $Q^{\vpi^{\text{old}}}$ on the RHS by applying the soft Bellman equation and the
bound in \cref{eq: bound_q}.
\end{proof}
Finally, combining \Cref{lemm: contraction q function} and \Cref{lem: policy_improvement}, we can prove \cref{thm: boltzman convergence}.
\begin{proof}[Proof of \Cref{thm: boltzman convergence}]
Note that since the reward, information gain, and entropy are bounded, $Q^{\vpi}(\vs, \va)$ is also bounded for all $\vpi \in \Pi$, $(\vs, \va) \in \gS \times \gA$. Furthermore, in \Cref{lem: policy_improvement} we show that the policy is monotonously improving. Therefore, the sequence of policy updates converges to a policy $\vpi^* \in \Pi$ such that
$J^{\vpi^{*}}(\vpi^*) \geq J^{\vpi^{*}}(\vpi)$
for all $\vpi \in \Pi$.
Applying the same argument as \citet[Theorem 1]{haarnoja2018soft}, we get $Q^{\vpi^{*}}(\vs, \va) \geq  Q^{\vpi}(\vs, \va)$ for all $\vpi \in \Pi$, $(\vs, \va) \in \gS \times \gA$.
\end{proof}
\section{\alg from the perspective of KL-minimization} \label{sec: kl minimization}
We take inspiration from \cite{hafner2020action} and study the Boltzmann exploration formulation from~\cref{subsec:boltzmanexplore} from the action perception and divergence minimization (APD) perspective. We denote with $\vtau = \{\vs_{t}, \va_{t}, r_{t}\}_{t\geq0}$ the trajectory in state, reward, and action space. 
The goal of the RL agent is to maximize the cumulative reward w.r.t.~the true system $\vf^*$. Therefore, a common and natural choice for a target/desired distribution for the trajectory $\vtau$ is~\citep{levine2018reinforcement, hafner2020action}
    \begin{equation*}
    p^*(\vtau) = \left[\prod_{t\geq0} p(\vs_{t+1}, r_t|\vs_t, \va_t, \vf^*) \right]\exp\left(\frac{1}{\alpha}\sum_{t\geq0} r_t\right).
\end{equation*}

However, $\vf^*$ is unknown in our case and we only have an estimate of the underlying distribution $p(\vf)$ from which it is sampled. 
Given a state-pair $(\vs_{t}, \va_t)$, we get the following distribution for $(\vs_{t+1}, r_t)$.
\begin{equation}
    p(\vs_{t+1}, r_t|\vs_t, \va_t) = \int p(\vs_{t+1}, r_t|\vs_t, \va_t, \vf) dp(\vf) = E_{\vf}\left[p(\vs_{t+1}, r_t|\vs_t, \va_t, \vf) \right]. 
\end{equation}
For a given policy $\vpi$, we can write the distribution of its trajectory via $\vtau^{\vpi}$ as
\begin{align}
    \hat{p}^{\vpi}(\vtau) &= \left[\prod_{t\geq0}  p(\vs_{t+1}, r_t|\vs_t, \va_t) \vpi(\va_t|\vs_t) \right] \notag \\
    &= p(\vs_1) \left[\prod_{t\geq0} E_{\vf}\left[p(\vs_{t+1}, r_t|\vs_t, \va_t, \vf) \right] \vpi(\va_t|\vs_t) \right].
\end{align}

A natural objective for deciding which policy to pick is to select $\vpi$ such that its resulting trajectory $\vtau^{\vpi}$ is close (in-distribution) to $\vtau^*$. That is,
\begin{equation}
  \vpi^* = \argmin_{\vpi}  \KL(\hat{p}^{\vpi}(\vtau) | p^*(\vtau)). \label{eq:policy_opt_max_ent_persp1}
\end{equation}
In the following, we break down the term from~\cref{eq:policy_opt_max_ent_persp1} and show that it results in the \alg objective.

\begin{align*}
    \KL(\hat{p}^{\vpi}(\vtau) | p^*(\vtau)) &= \E_{\vtau \sim \hat{p}^{\vpi}(\vtau)} \left[ - \log\left(\frac{p^*(\vtau)}{\hat{p}^{\vpi}(\vtau)}\right) \right] \\
    &= - \E_{\vtau \sim \hat{p}^{\vpi}(\vtau)}\left[  \sum_{t\geq0} \log(p(\vs_{t+1}, r_t|\vs_t, \va_t, \vf^*)) + \frac{1}{\alpha}\sum_{t\geq 0} r_t\right] \\
    &+ \E_{\vtau \sim \hat{p}^{\vpi}(\vtau)}\left[\sum_{t\geq0} \log\left(E_{\vf}\left[p(\vs_{t+1}, r_t|\vs_t, \va_t, \vf) \right]\right) + \log(\vpi(\va_t|\vs_t)) \right] \\
    &= \E_{\vtau \sim \hat{p}^{\vpi}(\vtau)}\left[\sum_{t\geq0} \log(\vpi(\va_t|\vs_t)) - \frac{1}{\alpha} r_t\right] \tag{1} \\
    &+ \E_{\vtau \sim \hat{p}^{\vpi}(\vtau)}\left[\sum_{t\geq0}\log\left(E_{\vf}\left[p(\vs_{t+1}, r_t|\vs_t, \va_t, \vf) \right]\right) - \log(p(\vs_{t+1}|\vs_t, \va_t, \vf^*)) \right] \tag{2}
\end{align*}
The term (1) is commonly minimized in max entropy RL methods such as \cite{haarnoja2018soft}. 
The term (2) has a very natural interpretation, which we discuss below
\begin{align*}
    &\E_{\vtau \sim \hat{p}^{\vpi}(\vtau)}\left[\sum_{t\geq0} \log\left(E_{\vf}\left[p(\vs_{t+1}, r_t|\vs_t, \va_t, \vf) \right]\right) - \log(p(\vs_{t+1}, r_t|\vs_t, \va_t, \vf^*)) \right] \\
    &= \E_{\vtau \sim \hat{p}^{\vpi}(\vtau)}\left[\sum_{t\geq 0}\log\left(p(\vs_{t+1}, r_t|\vs_t, \va_t) \right) - \log(p(\vs_{t+1}, r_t|\vs_t, \va_t, \vf^*))\right] \\
    &= \E_{\vtau \sim \hat{p}^{\vpi}(\vtau)}\left[\sum_{t\geq 0}\log\left(p(\vs_{t+1}, r_t|\vs_t, \va_t) \right) - \log(p(\vw_{t}))\right] \\
    &= \E_{\vtau \sim \hat{p}^{\vpi}(\vtau)}\left[\sum_{t\geq 0}
    H(\vw_t) -  H(\vs_{t+1}, r_t|\vs_t, \va_t).
    \right] \\
    &= \E_{\vtau \sim \hat{p}^{\vpi}(\vtau)}\left[\sum_{t\geq 0}
    H(\vs_{t+1}, r_t|\vs_t, \va_t, \vf) -  H(\vs_{t+1}, r_t|\vs_t, \va_t).
    \right]
\end{align*}
Here in the last step, we applied the equality that the entropy of $\vs_{t+1}, r_t$ given the previous state, input, and dynamics is the same as the entropy of $\vw_{t}$. In summary, we get
\begin{align*}
    &\E_{\vtau \sim \hat{p}^{\vpi}(\vtau)}\left[\sum_{t\geq 0} \log\left(E_{\vf}\left[p(\vs_{t+1}, r_t|\vs_t, \va_t, \vf) \right]\right) - \log(p(\vs_{t+1}, r_t|\vs_t, \va_t, \vf^*)) \right] \\
    &= \E_{\vtau \sim \hat{p}^{\vpi}(\vtau)}\left[\sum_{t\geq 0}
    -I(\vs_{t+1}, r_t; \vf |\vs_t, \va_t)\right].
\end{align*}

Therefore \cref{eq:policy_opt_max_ent_persp1} can be rewritten as
\begin{equation*}
      \vpi^* = \argmax_{\vpi}  \E_{\vtau \sim \hat{p}^{\vpi}(\vtau)}\left[\sum_{t\geq 0} r_t  - \alpha \log(\vpi(\vu_t|\vz_t)) + \alpha I(\vs_{t+1}, r_t; \vf |\vs_t, \va_t)\right].
\end{equation*}

Intuitively,  to bring our estimated distribution of the policy $\hat{p}^{\vpi}(\vtau)$ closer to the desired distribution $p^*(\vtau)$, we need to pick policies that maximize the underlying rewards while also gathering more information about the unknown elements of the MDP; the reward and dynamics model. While the former objective drives us in areas of high rewards under our estimated model, the latter objective ensures that our model gets more accurate over time.

\section{Additional Experiments} \label{sec: additional experiments}
We present additional experiments and ablations in this section.

\textbf{Covergence of intrinsic reward coefficient $\lambda$ for \algsac}:
In \cref{fig:saceipo weights} we report the tuned intrinsic reward coefficient $\lambda$ for \textsc{SACEipo} and \algsac ($\alpha_2$ for \algsac). As shown in \cref{fig:maxinforl_action_cost} in \cref{sec: experiments}, \textsc{SACEipo} obtains suboptimal performance compared to \algsac on the Pendulum, CartPole, and Walker tasks. We ascribe this to its tendency to underexplore, i.e. $\lambda  \to 0$ much faster and converges to the local optima. \Cref{fig:saceipo weights} validates our intuition since the intrinsic reward coefficient quickly decays to zero for \textsc{SACEipo} whereas for \algsac it does not.
\begin{figure}[th]
    \centering
    \includegraphics[width=\linewidth]{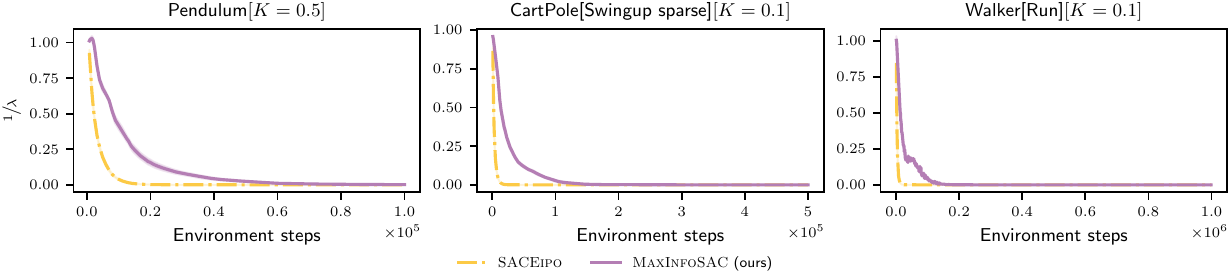}
    \caption{Evolution of the intrinsic reward coefficient $\lambda$ of \textsc{SACEipo} and \algsac over environment interaction. }
    \label{fig:saceipo weights}
\end{figure}

\textbf{\alg with REDQ}: In \cref{fig:maxinforl_redq} we combine \alg with REDQ. In all our tasks, REDQ does not obtain high rewards, we believe this is due to the high update to data ratio, which leads to the Q function converging to a sub-optimal solution. However, when combined with \alg, the resulting algorithm obtains considerably better performance and sample-efficiency. 
\begin{figure}[th]
    \centering
    \includegraphics[width=\linewidth]{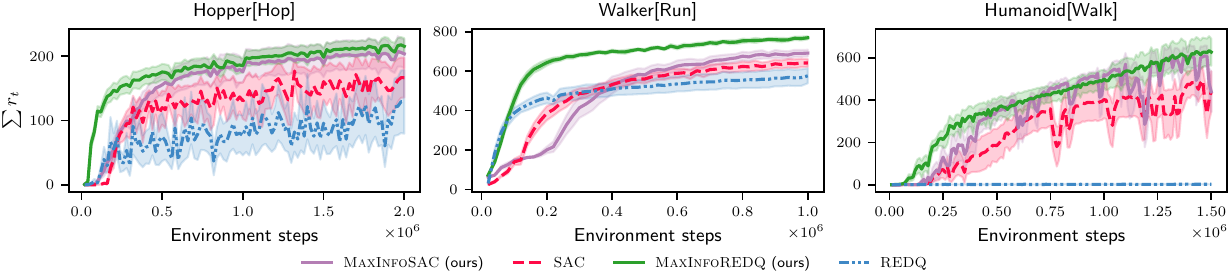}
    \caption{We combine \alg with REDQ (\textsc{MaxInfoREDQ}) and report the learning curves of SAC, REDQ, \algsac, and \textsc{MaxInfoREDQ}.}
    \label{fig:maxinforl_redq}
\end{figure}

\textbf{\alg with DrQv2 on additional DMC tasks}:
In \cref{fig:maxinforl_ddpg_ablation} we compare \algdrqddpg with DrQv2 and \algdrq. For \algdrqddpg, we use a fixed noise variance with \mbox{$\sigma=0.2$} and $\sigma=0.0$. We observe that on all tasks, \algdrqddpg performs better than DrQv2, even when no noise is added for exploration. This empirically demonstrates that already by just maximizing the information gain, we get good exploration.
\begin{figure}[th]
    \centering
    \includegraphics[width=\linewidth]{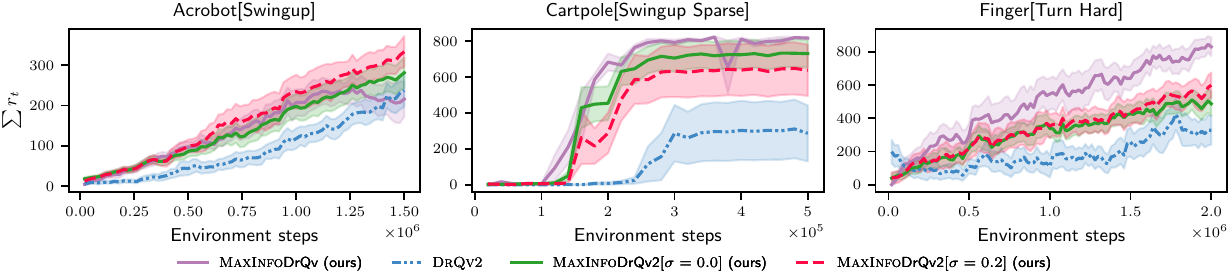}
    \caption{Learning curves of \algdrqddpg with different noise levels $\sigma \in \{0.0, 0.2\}$ compared to DrQv2 and \algdrq.}
    \label{fig:maxinforl_ddpg_ablation}
\end{figure}

In \cref{fig: maxinforl_ddpg_no_noise}, we evaluate the exploration behavior of \algdrqddpg without action noise on the DMC Humanoid Walk task. The figure shows that even without noise, \algdrqddpg outperforms DrQv2, while the noise-free DrQv2 struggles to achieve high rewards.
\begin{figure}[th]
    \centering
    \includegraphics[width=0.6\linewidth]{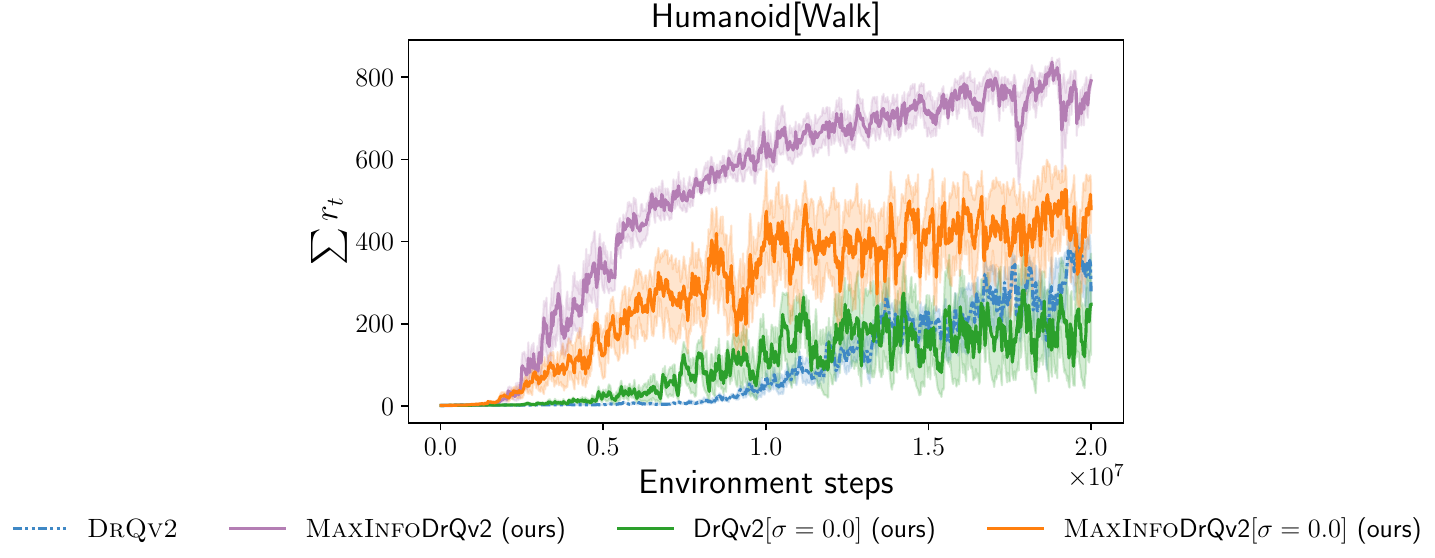}
    \caption{\algdrqddpg evaluated on the humanoid walk task with no action noise.}
    \label{fig: maxinforl_ddpg_no_noise}
\end{figure}

\textbf{\algsac with different intrinsic reward}:
We evaluate \algsac with RND as intrinsic reward. Moreover, we initialize a random NN model and train another model to predict the output of the random model as proposed in \citet{burda2018exploration}. We train an ensemble of NNs to learn the random model and use their disagreement as the intrinsic reward\footnote{\citet{burda2018exploration} use curiosity as the intrinsic reward but we observed that disagreement performed much better and robustly in our experiments.}. Effectively, we replace the information gain term in \cref{eq:vf_soft_q_updated} with the intrinsic reward derived from RND. In \cref{fig:maxinforl_rnd} we report the learning curves. From the figure, we conclude that \alg also works with RND as the exploration bonus. 
Moreover, while the standard \algsac performs better on these tasks, \algsac with RND performs significantly better than SAC.
This illustrates the generality of our approach.
\begin{figure}[th]
    \centering\includegraphics[width=\linewidth]{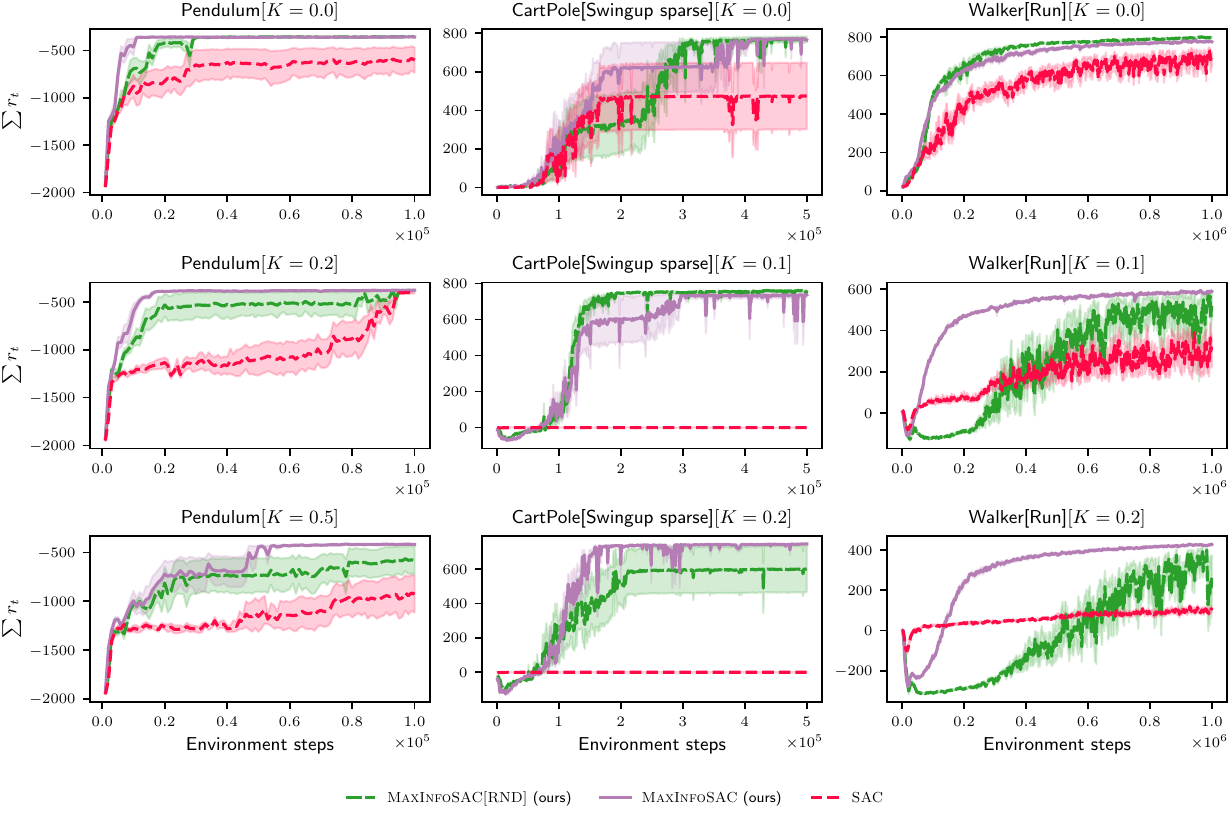}
    \caption{Learning curves of \algsac with RND as the intrinsic reward, instead of the information gain, compared to SAC and standard \algsac.}
    \label{fig:maxinforl_rnd}
\end{figure}

\textbf{Experiments with $\epsilon$--\alg}
We also evaluate our $\epsilon$--greedy version from \cref{sec: eps max info}, $\epsilon$--\alg, empirically. For the intrinsic reward, we use the disagreement~\citep{pathak2019self} in the ensemble models of $\vf^*$ and as the base algorithm, we use SAC.
In \cref{fig:maxinforl_eps_greedy} we compare it to \algsac and SAC and report the learning curves. We also evaluate $\epsilon$--\alg with different action costs, as in \cref{sec: experiments}, in \cref{fig:maxinforl_eps_greedy_action_cost}. For $\epsilon$ we specify a linear decay schedule (see~\cref{sec: Experiment details}). In the figures, we observe that $\epsilon$--\alg performs better than SAC and often on par to \algsac. However, $\epsilon$--\alg requires training two different actor-critic networks and manually specifying the schedule for $\epsilon$, as opposed to \algsac which automatically trades off the extrinsic and intrinsic rewards.
\begin{figure}[th]
    \centering
    \includegraphics[width=\linewidth]{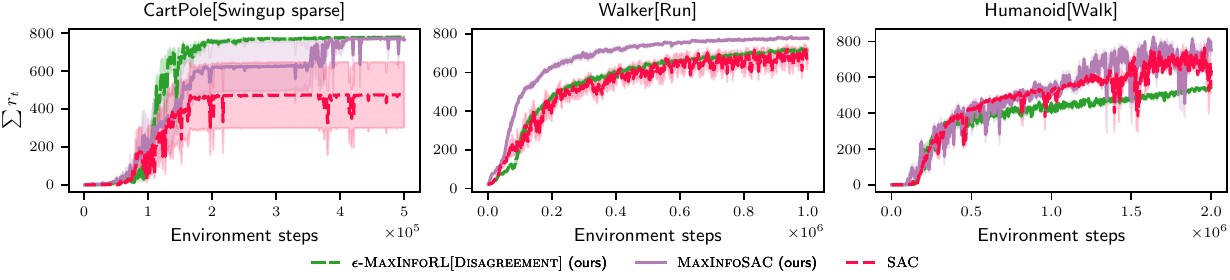}
    \caption{Learning curves of $\epsilon$--\alg with disagreement as the intrinsic reward, SAC and \algsac.}
    \label{fig:maxinforl_eps_greedy}
\end{figure}

\begin{figure}[th]
    \centering
    \includegraphics[width=\linewidth]{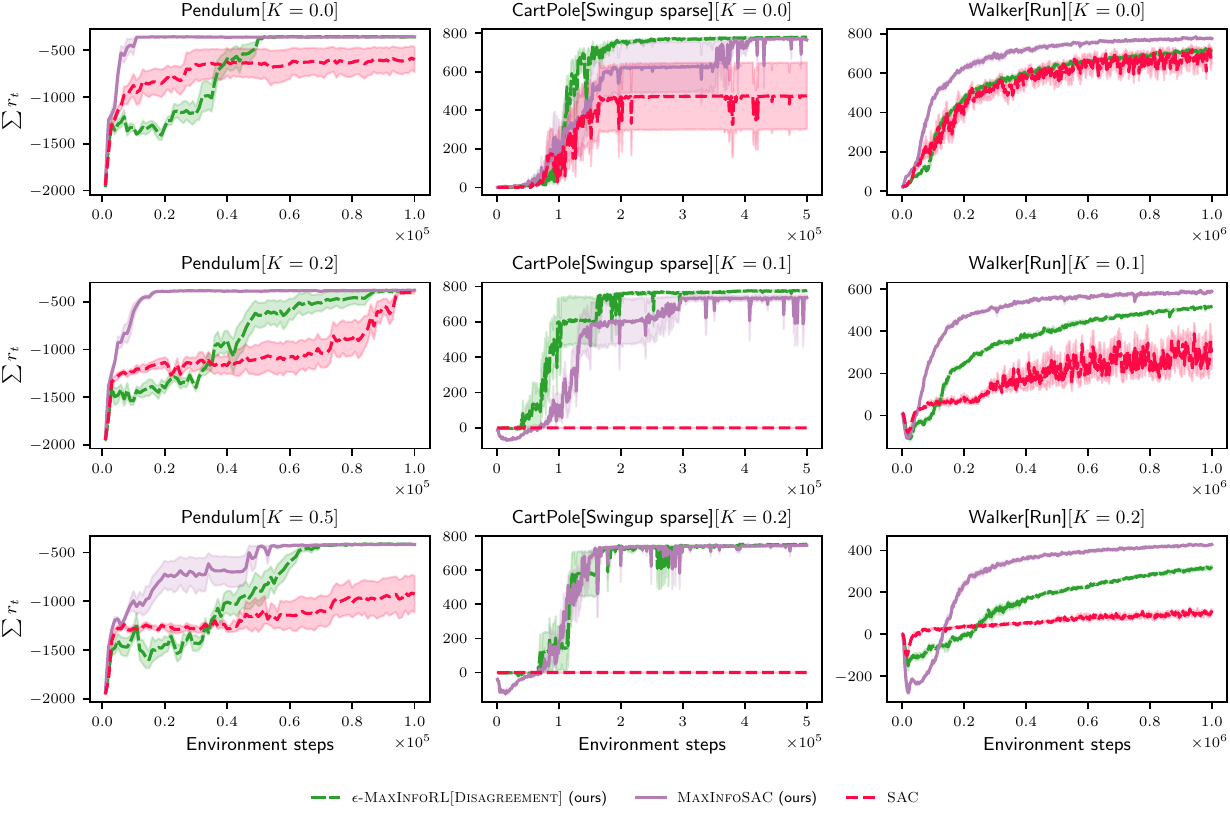}
    \caption{Learning curves of $\epsilon$--\alg with disagreement as the intrinsic reward, SAC and \algsac for varying levels of action costs $K$.}
    \label{fig:maxinforl_eps_greedy_action_cost}
\end{figure}

\textbf{Experiments with $\epsilon$--\alg and Curiosity}
In \cref{fig:maxinforl_eps_greedy_action_cost_curiosity} we compare $\epsilon$--\alg with disagreement and curiosity as intrinsic rewards. We observe that overall both perform better than SAC, with disagreement obtaining slightly better performance than curiosity.
\begin{figure}[th]
    \centering
    \includegraphics[width=\linewidth]{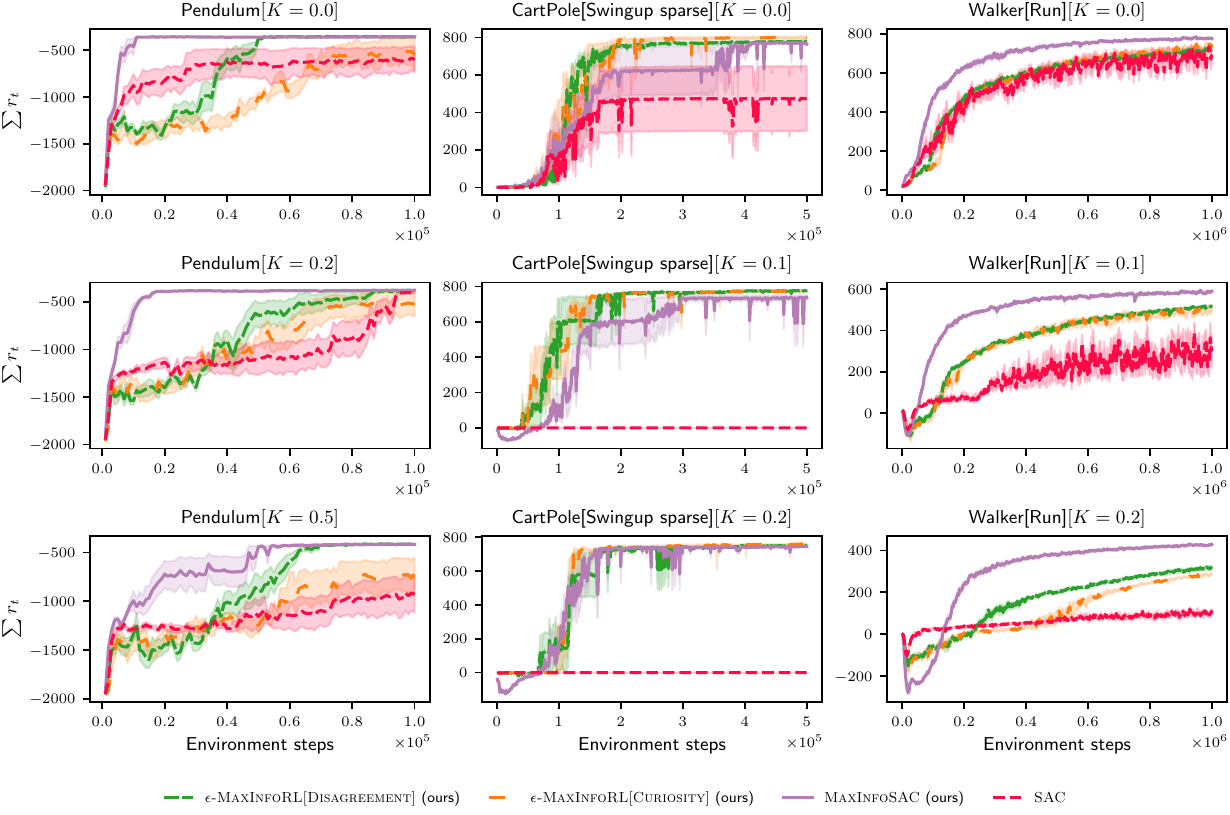}
    \caption{Learning curves of $\epsilon$--\alg with disagreement and curiosity as the intrinsic reward compared with SAC.}
    \label{fig:maxinforl_eps_greedy_action_cost_curiosity}
\end{figure}

\textbf{Experiments with $\epsilon$--\alg with a switching frequency of one between policies.}
\looseness=-1
In all the $\epsilon$--\alg experiments above, we switched between the extrinsic and intrinsic policies every 32 steps. 
Note that our proposed algorithm from \cref{sec: eps max info} is still the same and we simply use a different schedule for the $\epsilon$ parameter.
This is because we believe that instead of switching after every time step, following the extrinsic or intrinsic policies for several steps helps in exploration since we can collect longer trajectories of explorative data. 
In
\cref{fig:maxinforl_eps_greedy_action_cost_one_step} we test this hypothesis on the pendulum environment, where we observe that indeed the switching frequency of 32 steps performs better.
\begin{figure}[th]
    \centering
    \includegraphics[width=\linewidth]{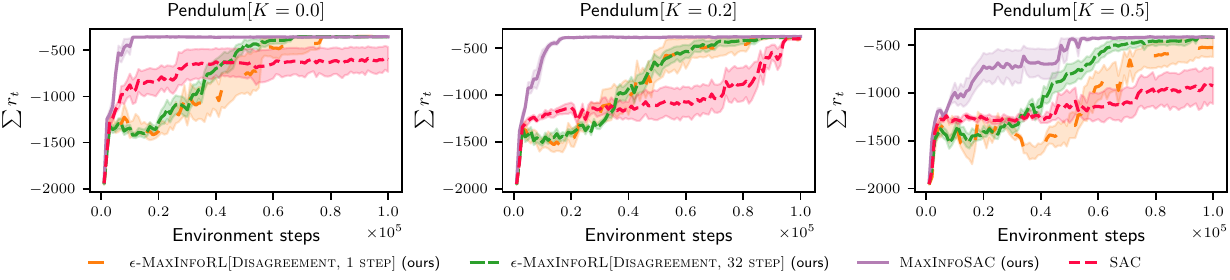}
    \caption{Learning curves of $\epsilon$--\alg with disagreement. We compare a version of $\epsilon$--\alg which switches between extrinsic and intrinsic policy every step with one which switches every 32 steps.}
    \label{fig:maxinforl_eps_greedy_action_cost_one_step}
\end{figure}

\textbf{Experiments with OAC~\citep{ciosek2019better}.}
In \cref{fig:maxinforl_oac_action_cost}, we compare OAC~\citep{ciosek2019better} an optimism based actor-critic algorithm with \algsac. Furthermore, we also test a \alg version of OAC, called \textsc{MaxInfoOAC}. As depicted in the figure, we observe that while OAC performs better than SAC, the \alg perform the best.
\begin{figure}[th]
    \centering
    \includegraphics[width=\linewidth]{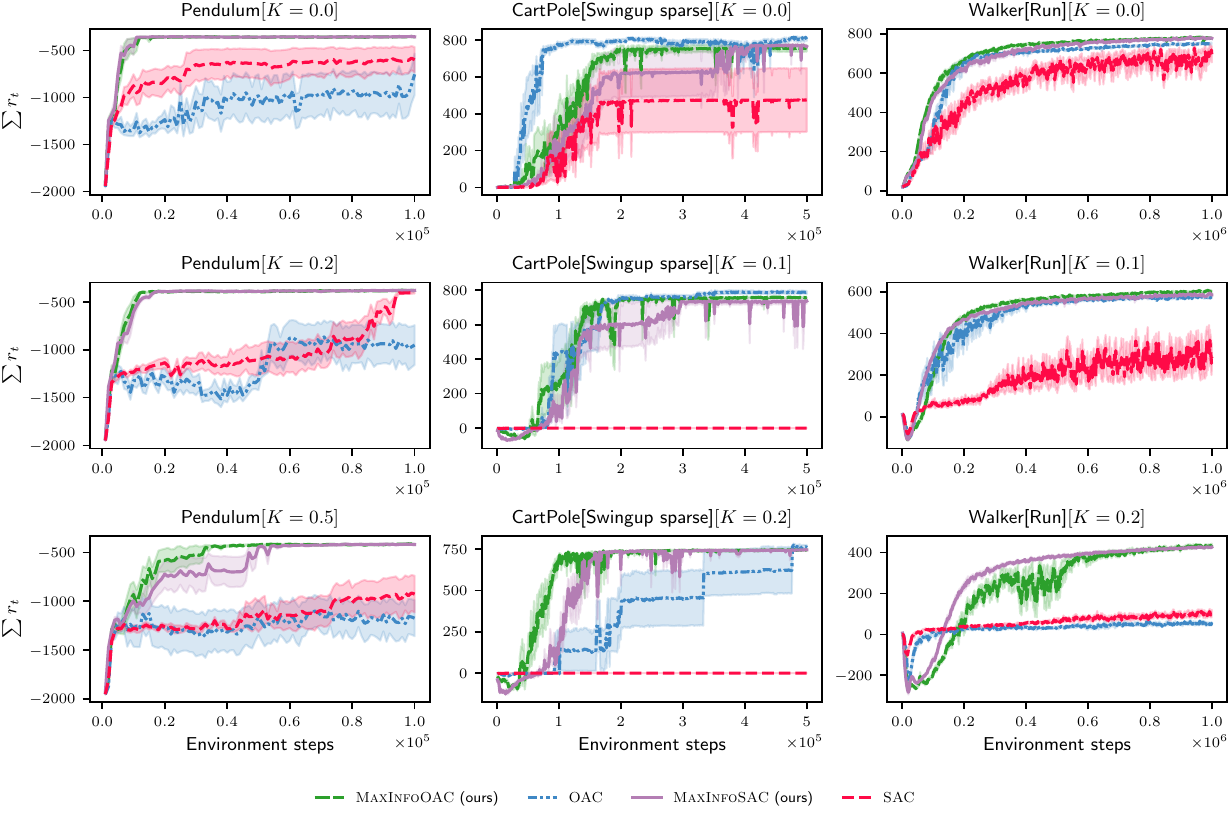}
    \caption{Learning curves of OAC compared with \algsac and a \alg version of OAC (\textsc{MaxInfoOAC}).}
    \label{fig:maxinforl_oac_action_cost}
\end{figure}

\textbf{Experiments with DrM~\citep{xu2023drm}.}
In \cref{fig:maxinforl_drm}, we evaluate DrM~\citep{xu2023drm} another exploration algorithm for visual control tasks. DrM uses dormant ratio-guided exploration to solve challenging visual control problems. In our experiments, we evaluate the version of DrM without the expectile regression and use the same exploration schedule proposed in the paper.
We combine DrM with \alg and report the performance on the humanoid and dog tasks from DMC. 
\begin{figure}[th]
    \centering
    \includegraphics[width=\linewidth]{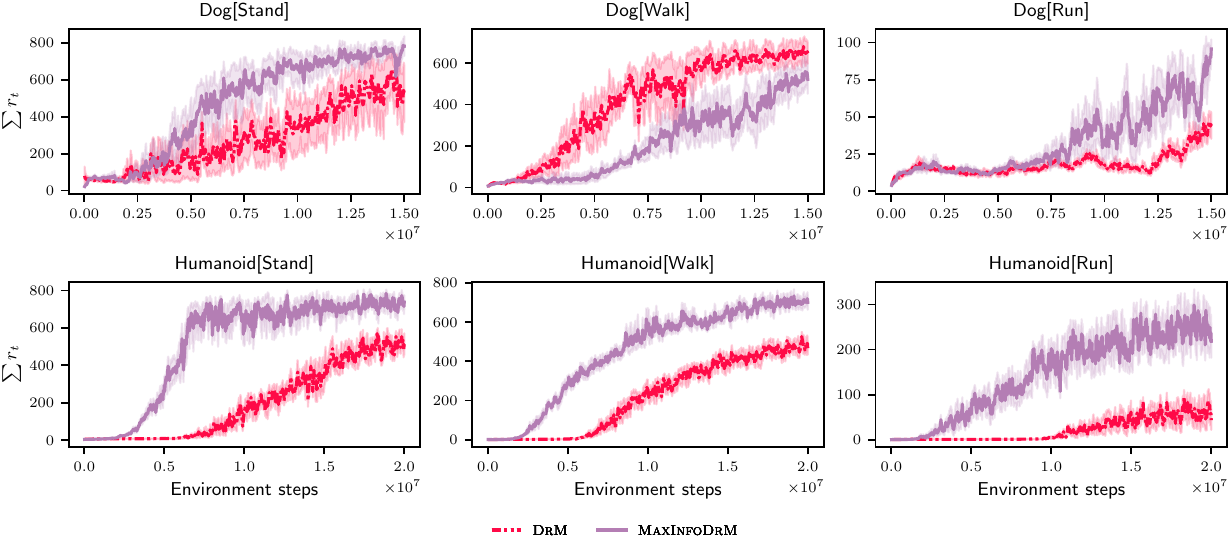}
    \caption{Learning curves of DrM compared with our version of DrM (\textsc{MaxInfoDrM}).}
    \label{fig:maxinforl_drm}
\end{figure}
\section{Experiment Details} \label{sec: Experiment details}
In this section, we present the experiment details of \alg and our baselines.

\paragraph{\algsac}
We train an ensemble of $P$ deterministic neural networks with mean squared error to predict the next state\footnote{Specifically we predict $\vs_{t+1}-\vs_t$.} and rewards. We evaluate the disagreement with $\vsigma_n (\vs, \va) = \text{Var}(\{\vf_{\phi_i}(\vs, \va)\}_{i \in \{1, \dots, P\}})$, where $\{\phi_i\}_{i \in \{1, \dots, P\}}$ are the parameters of each ensemble member. We quantify the information gain using the upper bound in \Cref{eq: entropy bound} and for the aleatoric uncertainty, we use $\sigma = 10^{-3}$. In principle, we could learn the aleatoric uncertainty, however, we found that this approach already works robustly. Lastly, as suggested by \citet{burda2018exploration} we normalize the intrinsic reward/information gain term when training the critic. 

A limitation of our work is that it requires training an ensemble of forward dynamics models to quantify the epistemic uncertainty. This is highlighted in \cref{tab:compute cost} where we compare the training time required to train SAC with \algsac. We report the numbers for our Stable-Baselines 3 (SB3)~\citep{stable-baselines3} (torch-based) implementation and JaxRL~\citep{jaxrl} (Jax-based) implementation.
\begin{table}[ht]
\centering
    \caption{Computation cost comparison for \algsac on NVIDIA GeForce RTX 2080 Ti}
    \label{tab:compute cost}
\begin{tabular}{l|l}
Algorithm & Training time for 100k environment interactions \\ \hline \
SAC (SB3)       & 16.96 min +/- 1.64317 min                       \\
\algsac (SB3)   & 39 min +/- 1 min                                \\
SAC (JaxRL)      & 5.6 min +/- 0.2min                              \\
\algsac (JaxRL)   & 7.3 min +/- 0.75                               
\end{tabular}
\end{table}

As the number suggests \algsac requires more time to train due to the additional computational overhead from learning the dynamics model. However, our JaxRL-based implementation already bridges the computational gap and is much faster. To further reduce the computational cost, we can consider cheaper methods for uncertainty quantification such as \citet{osband2023epistemic}.

\paragraph{Explore then exploit}
\looseness=-1
We train two different actor-critics, one for the extrinsic and one for the intrinsic reward. For the first $25\%$ of the episodes, we only use the intrinsic actor for data collection. For the remaining $75\%$, the extrinsic actor is used. As the base actor-critic algorithm we use SAC. Both the intrinsic and extrinsic actor-critics share the same data buffer. Therefore, the data collected during the exploration phase is also used to train the policy for the exploitation phase. For the intrinsic reward, we use the disagreement calculated as specified above for \algsac. We also evaluate this approach with curiosity as the intrinsic reward. Here we train a neural network model to predict the next state and reward and use the mean squared error in the prediction as the intrinsic reward.

\paragraph{\textsc{SACIntrinsic}}
Here we follow the same procedure as proposed in \citet{burda2018exploration}. We train two critics, one for the extrinsic reward and another one for the intrinsic rewards/exploration bonuses. We use the (normalized) information gain and policy entropy as the exploration bonuses. As suggested in \citet{burda2018exploration} the non-episodic returns are used in training the intrinsic critics. The policy is trained to maximize $\E_{\vs \sim \gD}[\E_{\va \sim \vpi(\cdot|\vs)}[Q^{\vpi_{\text{old}}}_{\text{extrinsic}}(\vs, \va) + Q^{\vpi_{\text{old}}}_{\text{intrinsic}}(\vs, \va)]]$, where $\gD$ is the data buffer. 

\paragraph{\textsc{SACEipo}}
We follow the same procedure as in \textsc{SACIntrinsic} in training the critics. For the policy, we maximize a weighted sum of the extrinsic and intrinsic critics, i.e., the policy $\vpi$ is trained to maximize $\E_{\vs \sim \gD}[\E_{\va \sim \vpi(\cdot|\vs)}[\lambda Q^{\vpi_{\text{old}}}_{\text{extrinsic}}(\vs, \va) +Q^{\vpi_{\text{old}}}_{\text{intrinsic}}(\vs, \va)]]$. We determine $\lambda$ using the extrinsic optimality constraint from \citet{chen2021randomized}. For the policy $\vpi$, the extrinsic optimality constraint is given by
\begin{equation*}
    \E_{\vs \sim \gD}[\E_{\va \sim \vpi(\cdot|\vs)}[Q^{\vpi}_{\text{extrinsic}}(\vs, \va)]] - \max_{\vpi' \in \Pi}\E_{\vs \in \gD}[\E_{\va \sim \vpi'(\cdot|\vs)}[Q^{\vpi'}_{\text{extrinsic}}(\vs, \va)]] = 0
\end{equation*}
Effectively, we encourage the policy $\vpi$ that maximizes the intrinsic and extrinsic rewards performs at least as well as the policy which only maximizes the extrinsic ones. 
To evaluate this constraint, we train another actor and critic that greedily maximizes the extrinsic reward only, i.e., we use the data buffer to learn the solution $\pi^*_E$ for $\max_{\vpi' \in \Pi}\E_{\vs \in \gD}[\E_{\va \sim \vpi'(\cdot|\vs)}[Q^{\vpi'}_{\text{extrinsic}}(\vs, \va)]]$. 
\begin{equation*}
    L(\lambda) = \lambda \left(\E_{\va \sim \vpi(\cdot|\vs)}[Q^{\vpi}_{\text{extrinsic}}(\vs, \va)]] - \E_{\vs \sim \gD}[\E_{\vs \in \gD}[\E_{\va \sim \pi^*_E(\cdot|\vs)}[Q^{\pi^*_E}_{\text{extrinsic}}(\vs, \va)]]\right)
\end{equation*}
Intuitively, if our constraint is not satisfied, i.e., $\E_{\vs \sim \gD}[\E_{\vs \in \gD}[\E_{\va \sim \pi^*_E(\cdot|\vs)}[Q^{\pi^*_E}_{\text{extrinsic}}(\vs, \va)]] > \E_{\vs \sim \gD}[\E_{\va \sim \vpi(\cdot|\vs)}[Q^{\vpi}_{\text{extrinsic}}(\vs, \va)]]$, we increase $\lambda$, effectively focusing more on the extrinsic reward. We update $\lambda$ with gradient descent on $L(\lambda)$, akin to the temperature parameters in \algsac.

\paragraph{\algdrq}
We add our information gain on top of DrQ using the same architecture and image augmentations from \citet{yarats2021image}. The policy and critic updates are identical to \algsac. However, since we deal with images, we train an ensemble of $P$ deterministic neural networks with mean squared error to predict the next state $\vs_{t+1}$, rewards $r_t$, and a compressed image embedding $e_t$ which we obtain from the encoder for the observation $o_t$. Effectively, we compress the output of the encoder into a $d_e$ dimensional embedding using max pooling and use that as a label for our ensemble training. Crucially, we want the ensemble to learn and explore all unknown components of the MDP, the dynamics, rewards and observation model. For the state $\vs$, we use the learned representation from the target critic. 

\paragraph{\algdrqddpg}
We use the same procedure for training the ensemble model as \algdrq. 
For \algdrqddpg, we separate the intrinsic and extrinsic reward critic. This is primarily because DrQv2 uses $n$--step returns for training the extrinsic critic. By separating the two critics, we can train the extrinsic critic with the $n$--step returns and the intrinsic with the standard $1$--step ones. This makes implementing the algorithm easier.

\paragraph{\textsc{MaxInfoSac[RND]}}
In \cref{fig:maxinforl_rnd} we evaluate \algsac with RND as intrinsic reward. We initialize a target neural network to predict an output embedding $\vy$ given $(\vs, \va)$. We train an ensemble of neural networks to learn the target network and use the disagreement as the intrinsic reward. The remaining training procedure is identical to \algsac.

\paragraph{$\epsilon$--\alg}
We also evaluate the $\epsilon$--\alg algorithm from \Cref{eq: eps_greedy}. We train two actor-critics similar to the explore and then exploit baselines. However, instead of exploring for a fixed amount of iterations initially, we specify a probability $\epsilon_t$ to alternate between the exploitation and exploration phases. Furthermore, changing the between policies after every time-step can lead to noisy trajectories. Therefore, we choose a frequency $\mathit{f} = 32$ for the switches. That is, we change the sampling strategy (explore or exploit) after a minimum of $\mathit{f} steps$.

\subsection{Algorithms}
We provide more details about \algsac, \algdrq, and \algdrqddpg. 
All three are off-policy actor-critic algorithms and learn two critics $Q_{\theta_1}$ and $Q_{\theta_2}$ to avoid overestimation bias. From here on, we denote with $\psi$ the parameters of the policy and $\bar{\theta}_{1, 2}, \bar{\psi}$ the parameters of the target critics and policy. 
\paragraph{Losses for \algsac and \algdrq}
We sample a trajectory $\vtau$ from $\gD$ and use the following for the critic and policy losses.

Critic loss:
\begin{align}
    J_{Q}(\theta) &= \frac{1}{2}\E_{\vtau \sim \gD}[(Q_{\theta}(\vs, \va) - y)^2] \notag \\
     y &= r + \gamma \min_{k \in \{1, 2\}}Q_{\Bar{\theta}_k}(\vs', \va') - \alpha_1 \log\vpi_{\psi}(\cdot|\vs') + \alpha_2 I_u(\vs', \va') \label{eq: critic loss sac}. 
\end{align}
where $\va' \sim \vpi_{\psi}(\cdot|\vs')$.

Policy loss:
\begin{align}
    J_{\vpi}(\psi) &= \E_{\vs \sim \gD}[\E_{\va \sim \vpi_{\psi}}[\alpha_1 \log\vpi_{\psi}(\cdot|\vs)  - \alpha_2 I_u(\vs, \va) - \min_{k \in \{1, 2\}}Q_{\theta_k}(\vs, \va)]]
     \label{eq: policy loss sac}. 
\end{align}

Entropy coefficient loss;
\begin{equation}
    J(\alpha_1) = \E_{\vs \sim \gD}[-\alpha_1 (\log\vpi_{\psi}(\cdot|\vs) + \Bar{H})]
     \label{eq: entropy loss sac}. 
\end{equation}

\paragraph{Critic loss for \algdrqddpg}
For \algdrqddpg we train separate critics for the intrinsic and extrinsic rewards. This allows us to use $n$--step returns for the extrinsic rewards while keeping a simple implementation, using $1$--step returns, for the intrinsic ones. Effectively, for computational efficiency, we adapt the critic network to output two heads, one for the extrinsic and one for the intrinsic reward.

Critic loss:
\begin{align}
    J_{Q}(\theta) &= \frac{1}{2}\E_{\vtau \sim \gD}[(Q^{\text{extrinsic}}_{\theta}(\vs, \va) - y^{\text{extrinsic}})^2 +  (Q^{\text{intrinsic}}_{\theta}(\vs, \va) - y^{\text{intrinsic}})^2] \notag \\
     y^{\text{extrinsic}} &= \sum^{n}_{l=0} \gamma^{l} r_{l} + \gamma^n \min_{k \in \{1, 2\}}Q^{\text{extrinsic}}_{\Bar{\theta}_k}(\vs_{n}, \va_{n}) \notag \\
     y^{\text{intrinsic}} &=I_u(\vs, \va) + \gamma \min_{k \in \{1, 2\}}Q^{\text{intrinsic}}_{\Bar{\theta}_k} (\vs', \va').  \label{eq: loss critic drqv2}
\end{align}
where $\va_n \sim \vpi_{\psi}(\cdot|\vs_n)$ and $\va' \sim \vpi_{\psi}(\cdot|\vs')$.

Policy loss:
\begin{align}
    J_{\vpi}(\psi) &= \E_{\vs \sim \gD}\left[\E_{\va \sim \vpi_{\psi}}\left[- \left(\min_{k \in \{1, 2\}}Q_{\theta_k}^{\text{extrinsic}}(\vs, \va)  + \alpha_2 \min_{k \in \{1, 2\}}Q_{\theta_k}^{\text{intrinsic}}(\vs, \va)\right)\right]\right]
     \label{eq: policy loss drqv2}. 
\end{align}

Loss for the information gain coefficient
\begin{equation}
    J(\alpha_2) = \alpha_2\left(\E_{\vs \sim \gD}[\E_{\va \sim \vpi_{\psi}}[I_u(\vs, \va)] - \E_{\bar{\va} \sim \bar{\vpi}_{\psi}}[I_u(\vs, \bar{\va})]]\right)
     \label{eq: infogain loss}. 
\end{equation}

Loss of ensemble model for \algsac. Let $\Phi = \{\phi_i\}_{i \in \{1, \dots P\}}$
\begin{align}
    J(\Phi) &= \E_{\vtau \sim \gD}\left[\sum^{P}_{i=1}\norm{\vf_{\phi_i}(\vs, \va) - y}^2\right] \\
    y & = \left[(\vs' - \vs)^{\top}, r\right]
     \label{eq: ensemble loss sac}. 
\end{align}

Loss of ensemble model for \algdrq and \algdrqddpg.
\begin{align}
    J(\Phi) &= \E_{\vtau \sim \gD}\left[\sum^{P}_{i=1}\norm{\vf_{\phi_i}(\vs, \va) - y}^2\right] \\
    y & = \left[\vs'^{\top}, r, e^{\top}\right]
     \label{eq: ensemble loss drq}. 
\end{align}

In \Cref{alg: main algorithm structure} we present the main structure of the algorithms. We omit algorithm-specific details of DrQ, DrQv2 such as noise scheduling and image augmentation in \Cref{alg: main algorithm structure} and refer the reader to the respective papers for those details~\citep{yarats2021image, yarats2021mastering}.
\begin{algorithm}[ht]
    \caption{Algorithm structure for \algsac, \algdrq, \algdrqddpg}
        \label{alg:safe-exploration}
    \begin{algorithmic}[]
        \State {\textbf{Init:}}{ $\theta_1$, $\theta_2$, $\psi$}
        \State $\bar{\theta}_1 \leftarrow \theta_1$, $\bar{\theta}_2 \leftarrow \theta_2$, $\bar{\psi} \leftarrow \psi$ \algorithmiccomment{Initialize target}
        \State $\gD = \emptyset$ 
         \For{iterations $n=1, \ldots, n$}
         \For{each environment step}
         \State $\va_t \sim \vpi_{\psi}(\cdot|\vs_t)$ \algorithmiccomment{Sample action}
         \State $\vs_{t+1}, r_t \sim p(\vs_{t+1}, r_t|\vs_t, \va_t)$ \algorithmiccomment{Observe state and reward}
         \State $\gD \leftarrow \gD \cup \{\vs_t, \va_t, \vs_{t+1}, \va_t\}$
         \EndFor
        \For{each gradient step}
        \State Update critics with stochastic gradient descent (SGD) on $J_{Q}(\vtheta_1) + J_{Q}(\vtheta_2)$
        \State Update critics with SGD on $J_{\vpi}(\psi)$
        \State Update temperatures with SGD on $J(\alpha_1), J(\alpha_2)$\footnote{$\alpha_1$ is not used in \algdrqddpg}
        \State Update ensemble with SGD on $J(\Phi)$.
        \State Update $\bar{\theta}_{1, 2}, \psi$ \algorithmiccomment{Policy update target updates}
        \EndFor
        \EndFor
    \end{algorithmic}
    \label{alg: main algorithm structure}
\end{algorithm}
\subsection{Hyperparameters}
\begin{table}[H]
\centering
    \caption{Hyperparameters for results in~\Cref{sec: experiments}.}
    \label{tab:environment_hyperparams}
\begin{adjustbox}{max width=\linewidth}\begin{threeparttable}
\begin{tabular}{|c|cccccc}
Environment    & \multicolumn{1}{c|}{Learning rate}                                                                   & \multicolumn{1}{c|}{\begin{tabular}[c]{@{}c@{}}Policy/Critic \\ Architecture\end{tabular}} & \multicolumn{1}{c|}{Model architcture} & \multicolumn{1}{c|}{\begin{tabular}[c]{@{}c@{}}Polyak\\ Coefficient\end{tabular}} & \multicolumn{1}{c|}{\begin{tabular}[c]{@{}c@{}}action \\ repeat\end{tabular}} & \begin{tabular}[c]{@{}c@{}}feature \\ and embedding dim\end{tabular} \\ \hline
               & \multicolumn{6}{c}{State Based Tasks}                                                                                                                                                                                                                                                                                                                                                                                                                                                 \\ \cline{2-7} 
DMC and Gym    & $3 \times 10^{-4}$                                                                                                 & $(256, 256)$                                                                                 &                                        & $0.005$                                                                             & \begin{tabular}[c]{@{}c@{}}$1$ for Gym\\ $2$ for DMC\end{tabular}                 & -                                                                    \\
Humanoid Bench & $5\times 10^{-5}$                                                                                                 &                                                                                            &                                        & $0.005$                                                                             &                                                                               & -                                                                    \\ \cline{2-7} 
               & \multicolumn{6}{c}{Visual Control tasks}                                                                                                                                                                                                                                                                                                                                                                                                                                              \\ \cline{2-7} 
DMC            & \begin{tabular}[c]{@{}c@{}}$3 \times 10^{-4}$ for \algdrq, else $10^{-4}$\\ $10^{-4}$ for the model and $\alpha_2$\end{tabular} & (256, 256)                                                                                 & (256, 256)                             & \begin{tabular}[c]{@{}c@{}}$0.005$ for \algdrq, \\ else $10^{-4}$ \end{tabular}             & 2                                                                             & \begin{tabular}[c]{@{}c@{}}Feature: 50\\ Embedding: 32\end{tabular}  \\
Humanoid and Dog tasks (DrM) & \begin{tabular}[c]{@{}c@{}}$8\times 10^{-5}$, $8\times10^{-5}$ for the model and $\alpha_2$\end{tabular}                     & $(1024, 1024)$                                                                               & $(1024, 1024)$                            & $0.01$                                                                              & $2$                                                                             & \begin{tabular}[c]{@{}c@{}}Feature: $100$\\ Embedding: $32$\end{tabular}
\end{tabular}
\end{threeparttable}\end{adjustbox}
\end{table}

\end{document}